%% file: submit_icml2026_v1_aid_itd_majorfix_no_f2sa_article.tex
\documentclass[11pt]{article}
\usepackage[a4paper,margin=1in]{geometry}
\usepackage[T1]{fontenc}
\usepackage[utf8]{inputenc}
\input{math_commands.tex}

\usepackage{microtype}
\usepackage{graphicx}
\usepackage{subcaption}
\usepackage{booktabs} % for professional tables
\usepackage[most]{tcolorbox}
\usepackage{xcolor}
\usepackage[round,authoryear]{natbib}
\usepackage{algorithm}
\usepackage{algorithmic}

\tcbset{
  myframe/.style={ % 定义一个名为 mytheoremframe 的样式
    colback=white, % 框体内部的背景颜色
    colframe=blue, % 框线的颜色
    boxrule=1pt, % 框线的粗细
    arc=4pt, % 框体的圆角半径
    outer arc=4pt, % 外部圆角半径（如果需要）
    top=3mm, bottom=3mm, left=3mm, right=3mm, % 内边距
  }
}

\usepackage{hyperref}
\hypersetup{hidelinks}

\usepackage{amsmath}
\usepackage{amssymb}
\usepackage{mathtools}
\usepackage{amsthm}

\usepackage[capitalize,noabbrev]{cleveref}

\theoremstyle{plain}
\newtheorem{theorem}{Theorem}[section]

\newtheorem{lemma}[theorem]{Lemma}

\theoremstyle{definition}
\newtheorem{definition}[theorem]{Definition}
\newtheorem{assumption}[theorem]{Assumption}
\theoremstyle{remark}
\newtheorem{remark}[theorem]{Remark}

\title{Sharper Analysis of Single-Loop Methods for Bilevel Optimization}

\author{%
Yubo Zhou\textsuperscript{1}\quad
Jun Shu\textsuperscript{1}\quad
Luo Luo\textsuperscript{2}\quad
Junmin Liu\textsuperscript{1}\\
Deyu Meng\textsuperscript{1}\quad
Guang Dai\textsuperscript{3}\quad
Haishan Ye\textsuperscript{4}\\[0.8em]
\small \textsuperscript{1}School of Mathematics and Statistics, Xi'an Jiaotong University\\
\small \textsuperscript{2}School of Data Science, Fudan University\\
\small \textsuperscript{3}SGIT AI Lab, State Grid Corporation of China\\
\small \textsuperscript{4}Center for Intelligent Decision-Making and Machine Learning, School of Management,\\
\small Xi'an Jiaotong University
}

\date{}

\begin{document}
\maketitle

\begin{abstract}
  Bilevel optimization underpins many machine learning applications, including hyperparameter optimization, meta-learning, neural architecture search, and reinforcement learning. While hypergradient-based methods have advanced significantly, a gap persists between theoretical guarantees and practical single-loop implementations required for efficiency. We bridge this gap by establishing sharper convergence results for single-loop approximate implicit differentiation (AID) and iterative differentiation (ITD) methods, leveraging our proposed analytical framework, decoupled norm analysis (DNA). For AID, we improve the convergence rate from $\mathcal{O}(\kappa^6/K)$ to $\mathcal{O}(\kappa^5/K)$, where $\kappa$ is the condition number of the inner-level problem. For ITD, we prove that the asymptotic error is $\mathcal{O}(\kappa^2)$, exactly matching the known lower bound and improving upon the previous $\mathcal{O}(\kappa^3)$ guarantee. Numerical experiments on synthetic and real tasks corroborate our theoretical findings.
\end{abstract}

\section{Introduction}
Bilevel optimization has attracted extensive attention in various applications of machine learning, including hyperparameter optimization \citep{maclaurin2015gradient, franceschi2017forward, shaban2019truncated, shen2024seal}, meta-learning \citep{chen2017learning, finn2017model, franceschi2018bilevel}, neural architecture search \citep{liu2018darts, he2020milenas}, and reinforcement learning \citep{zhang2020bi, wang2020global, shen2025principled}. Bilevel optimization corresponds to solving one optimization problem subject to constraints defined by another optimization problem. In this paper, we focus on the following bilevel optimization problem:
\begin{align}
\label{eq:bo}
    &\min_{x\in \sR^m} \Phi(x) = f(x, y^*(x)),\notag\\
    &\mbox{s.t.} \quad y^*(x) = \argmin_{y\in \sR^n} g(x, y),
\end{align}
where the outer- and inner-level functions $f$ and $g$ are both jointly continuously differentiable on $\sR^m\times\sR^n$. We focus on the setting where $g$ is strongly convex with respect to (w.r.t.) the inner-level variable $y$, which can guarantee the uniqueness of the inner solution \citep{chen2024finding}.

Hypergradient-based algorithms have recently gained significant attention for their balance of simplicity and efficiency. Two prominent approaches are approximate implicit differentiation (AID) \citep{domke2012generic, pedregosa2016hyperparameter, ghadimi2018approximation, grazzi2020iteration, ji2021bilevel} and iterative differentiation (ITD) \citep{franceschi2017forward, shaban2019truncated, grazzi2020iteration, ji2021bilevel, liu2021towards}. The key distinction lies in how they estimate the hypergradient $\nabla \Phi(x)$: AID leverages the implicit function theorem, while ITD applies automatic differentiation (see Section \ref{sec:alg}). Despite this difference, both methods require solving the inner problem to obtain the optimal solution $y^*$. In practice, however, closed-form solutions are rarely available, and one typically resorts to gradient descent to compute an approximate solution $\hat{y}$.

Most theoretical studies of bilevel optimization analyze algorithms that employ multi-loop updates (multi-step gradient descent) for the inner problem and linear-system \citep{ghadimi2018approximation, ji2021bilevel, dong2025efficient, fang2025qnbo}. In contrast, practical algorithms overwhelmingly adopt \textbf{single-loop} updates, where only one inner update is performed per outer iteration. The main appeal of single-loop methods is computational efficiency: they significantly reduce training cost while maintaining competitive performance. This design has become standard across a wide range of applications. For instance, in neural architecture search, DARTS \citep{liu2018darts} updates the network parameters ($y$) via single-loop while optimizing architecture coefficients ($x$). In few-shot meta-learning, MAML \citep{finn2017model} applies single-loop adaptation to task-specific parameters. In data reweighting for imbalanced or noisy samples, methods such as \citet{ren2018learning, shu2019meta} also rely on single-loop updates. These examples highlight a critical question: \textit{despite the clear practical advantages of single-loop algorithms, can their theoretical properties be rigorously guaranteed?}

% Recently, \citet{liu2024moreau} propose MEHA, a Moreau-envelope-based single-loop method with convergence rate $O(1/K^{1/2-p} + 1/K^{p})$, where $K$ is the number of outer iterations and $p \in (0,1/2)$. \citet{kwon2023fully} design $\text{F}^3\text{SA}$ by incorporating momentum, achieving a rate of $O(K^{-2/3})$. However, these single-loop methods remain slower than AID and ITD, both of which can reach $O(K^{-1})$ as shown in Table \ref{tab:intro}. Motivated by this gap, we focus on the AID and ITD methods and seek sharper analyses for their single-loop variants.

Along similar lines, \citet{ji2022will} analyze the single-loop structure in bilevel optimization and establish corresponding theoretical results. For AID, \citet{ji2022will} establish a convergence of $\gO(\kappa^6/K)$ in the single-loop setting, where $\kappa = \frac{L}{\mu}$ denotes the condition number ($L$ and $\mu$ are the gradient Lipschitz and strong convexity constants defined respectively in Assumptions \ref{assum-1} and \ref{assum-3}). For ITD, \citet{ji2022will} show that single-loop suffers from an inherent error of order $\gO(\kappa^3)$, still leaving a gap of $\alpha \mu$ (with $\alpha$ the inner-level step size) from the fundamental lower bound. They identify closing this gap as an open problem.

Although \citet{ji2022will} establish theoretical convergence bounds for the single-loop AID and ITD algorithms, we observe a non-negligible gap between their theoretical guarantees and the empirical convergence behavior (see the experiments in Section \ref{sec:exp-main}). This gap primarily stems from their analysis, which directly upper-bounds the squared error norms in a loose manner, thereby inducing an amplified dependence on the condition number $\kappa$. To address this key issue, we propose a novel analytical framework, termed \textit{Decoupled Norm Analysis} (DNA), which breaks away from the limitations of the existing proof template for single-loop methods. Specifically, we first control the linear error norms of individual variables and then square them afterward, effectively decoupling the analysis. This more refined treatment avoids excessive looseness and thus yields sharper bounds, enabling a more faithful characterization of the theoretical properties of the algorithms.

Our main contributions can be summarized as follows.
\begin{itemize}
    \item For AID, via a refined analysis and the novel analytical methodology DNA, we show that the single-loop AID algorithm can achieve a convergence rate of $\gO(\kappa^5/K)$, thereby providing a more accurate guarantee for bilevel optimization tasks where previous guarantees of $\gO(\kappa^6/K)$ limited reliability.
    \item For ITD, the single-loop ITD algorithm can achieve a convergence error on the order of $\gO(\kappa^2)$, which exactly matches the lower bound established by \citet{ji2022will}. This result thus establishes the theoretical optimality of single-loop ITD method and provides a solid reliability guarantee for their practical deployment.
    \item For DNA, the analytical tool developed for single-loop methods, we show that a linear-norm recursion followed by delayed squaring yields sharper convergence results for both AID and ITD, as summarized in Table \ref{tab:intro}. This provides a unified explanation for the improved condition-number dependence of the two hypergradient-based algorithms.
\end{itemize}

\section{Related Work}
\begin{table}[!t]
    \centering
    \caption{Comparison of computational complexities of both single-loop AID-based and ITD-based algorithms for finding an $\epsilon$-stationary point. For the last three columns, `N/A’ means that the complexities to achieve an $\epsilon$-accuracy are not measurable
due to the nonvanishing convergence error. MV($\epsilon$): the total number of Jacobian- and Hessian-vector product computations. Gc($\epsilon$): the total number of gradient computations.}
    \label{tab:intro}
    % \scalebox{0.8}{
    \resizebox{\linewidth}{!}{
    \begin{tabular}{lccc}
     \hline
     Algorithms & Convergence rate & \textbf{MV}($\epsilon$) & \textbf{Gc}($\epsilon$) \\
     \hline
     \hline
     % \textbf{with Single-Loop}
    %  MEHA \citep{liu2024moreau} & $\gO(1/)$
    AID \citep{ji2022will} & $\gO({\kappa^6}/{K})$ & $\gO(\kappa^6\epsilon^{-1})$& $\gO(\kappa^6\epsilon^{-1})$ \\
     % \hline
     AID (this paper) & $\gO({\kappa^5}/{K})$ & $\gO(\kappa^5\epsilon^{-1})$& $\gO(\kappa^5\epsilon^{-1})$ \\
     
     \hline
     ITD \citep{ji2022will} & $\gO({\kappa^3}/{K}+\kappa^3)$ & N/A & N/A \\
     % \hline
     ITD (this paper) & $\gO({\kappa^3}/{K}+\kappa^2)$ & N/A & N/A \\
     Lower bound of ITD & $\Omega(\kappa^2)$ & N/A & N/A \\
     \hline
    \end{tabular}}
\end{table}
\textbf{Hypergradient-based bilevel optimization.}
A variety of hypergradient-based bilevel algorithms have been proposed, differing mainly in how they estimate hypergradients. Methods based on approximate implicit differentiation (AID) \citep{domke2012generic, pedregosa2016hyperparameter, ghadimi2018approximation, grazzi2020iteration, ji2021bilevel} estimate the product of the inverse hessian and a vector by solving linear systems with efficient iterative solvers. In contrast, iterative differentiation (ITD) methods \citep{maclaurin2015gradient, franceschi2017forward, shaban2019truncated, liu2021towards} compute hypergradients by backpropagating through the inner optimization trajectory. The convergence properties of AID- and ITD-based algorithms have been the subject of extensive study. For example, \citet{ghadimi2018approximation} and \citet{ji2021bilevel} analyzed the convergence rates and complexities of both approaches, while \citet{ji2022will} provided a unified framework covering different inner-loop choices and established lower bounds on the inherent error of ITD. Despite this progress, a notable gap remains between the convergence rate of the single-loop and multi-loop algorithms. Motivated by this gap, our work develops sharper convergence guarantees for single-loop methods, which are widely used in practice. Compared with \citet{ji2022will}, our analysis for AID achieves an improved convergence order, while for ITD we refine the upper bound on the inherent error to match its known lower bound.

\textbf{Gradient-based bilevel optimization.} In recent years, some first-order gradient-based bilevel optimization methods have also attracted attention. \citet{chen2025near} proposed an algorithm that achieves near-optimal complexity under the nonconvex–strongly convex setting; however, it still requires $O(\kappa \log(\lambda \kappa))$ inner iterations, where the penalty parameter $\lambda=O(\kappa^3)$ can be large. \citet{liu2024moreau} proposed the single-loop MEHA method based on the Moreau envelope and established a convergence rate of $O(1/K^{1/2-p}+1/K^p)$ for $p\in(0,1/2)$. By introducing momentum, \citet{kwon2023fully} designed the single-loop $\text{F}^3\text{SA}$ method and obtained an $O(K^{-2/3})$ rate. Because these rates are slower than those available for hypergradient-based methods, our theoretical development focuses on single-loop AID and ITD.

\textbf{The single-loop bilevel optimization algorithms.} The single-loop methods have shown potential in many applications. In few-shot meta-learning, MAML \citep{finn2017model}, as a classic method, performs single-step gradient descent on the support set for multiple tasks in the inner-level, retaining the iteration path, while the outer-level updates the network's initial values using the query set. In hyperparameter optimization, sample reweighting is a widely used application of bilevel optimization algorithms \citep{ren2018learning, shu2019meta, wang2024relational}, as bilevel optimization can efficiently assign different weights to each sample. Such methods typically use the training set in the inner-level to perform single-step gradient descent to optimize model parameters, and the validation set in the outer loop to optimize sample weights or weighted networks. In neural architecture search, DARTS \citep{liu2018darts} method uses a one-step update in the inner-level to update the model, and the outer-level optimizes the architecture using validation data. It is worth noting that most of these algorithms achieve efficiency by single-loop, which is also crucial for the large-scale practice of bilevel optimization techniques \citep{choe2023making, shen2024seal}. Therefore, in this work, we focus on the single-loop bilevel optimization algorithms, consistent with practical applications, and are committed to establishing sharper convergence guarantees for these algorithms.

\section{Algorithms}\label{sec:alg}
In this section, we introduce two popular bilevel optimization algorithms to solve problem (\ref{eq:bo}). It is worth noting that we provide the single-loop algorithms, as this aligns with practical choices in related applications.
\begin{algorithm}[t]
	\caption{Single-Loop AID-based bilevel optimization algorithm}
%	\small
	\label{alg:bo-aid}
	\begin{algorithmic}[1]
		\STATE {\bfseries Input:}  Learning rates $\alpha, \beta,\eta >0$, initializations $x_0, y_0,v_0$.
%	        \STATE {\bfseries} 
		\FOR{$k=0,1,2,...,K-1$}
		\STATE{Set $y_k^0 = \hy_{k-1} \mbox{ if }\; k> 0$ and $y_0$ otherwise  \textbf{\em (warm start initialization)}} 
		% \vspace{0.05cm}
		\STATE{Update $\hy_k = y_k^{0}-\alpha \nabla_y g(x_k,y_k^{0}) $}
\STATE{Set $v_k^0 = \hv_{k-1} \mbox{ if }\; k> 0$ and $v_0$ otherwise  \textbf{\em (warm start initialization)}} 
\STATE{Update $\hv_k = (I-\eta\nabla_y^2 g(x_k, \hy_k))v_k^{0}+\eta\nabla_y f(x_k, \hy_k)$}
\STATE Compute hypergradient $\widehat\nabla \Phi(x_k)= \nabla_x f(x_k,\hy_k) -\nabla_{xy}^2 g(x_k,\hy_k)\hv_k$
            
                 \STATE{Update $x_{k+1}=x_k- \beta \widehat\nabla \Phi(x_k) $}
		\ENDFOR
	\end{algorithmic}
	\end{algorithm}
\subsection{AID-based Bilevel Optimization Algorithm}
We provide the single-loop AID-based bilevel optimization algorithm (for simplicity, hereafter referred to as AID) in Algorithm \ref{alg:bo-aid}. In each outer-level iteration $k$, AID first performs one step of gradient descent on the inner-level function $g(x,y)$ to find a point $\hy_k$ that approximates $y_k^*$, where $y_k^*$ denotes $\argmin_y g(x_k, y)$. Moreover, to accelerate the practical training process, AID usually adopts a warm-start strategy. In other words, the initial value $y_k^0$ of the inner-level problem at iteration $k$ is set to the updated value $\hy_{k-1}$ from iteration $k-1$.

In the outer-level, AID first obtain $\hv_k$ via solving a linear system $\nabla_y^2 g(x_k,\hy_k) v =\nabla_y f(x_k,\hy_k)$ by one step of gradient descent starting form $v_k^0$, and then AID can estimate the gradient $\nabla\Phi(x_k)=\nabla_x f(x_k,y_k^*)-\nabla_{xy}^2 g(x_k, y_k^*)\hv_k$ of the outer-level function w.r.t. $x$ (called hypergradient) by the form of $\widehat{\nabla}\Phi(x_k)=\nabla_x f(x_k,\hy_k)-\nabla_{xy}^2 g(x_k, \hy_k)\hv_k$.

% As shown in \citet{domke2012generic, grazzi2020iteration}, the computation of $\widehat{\nabla}\Phi(x_k)$ only involves Hessian-vector products in solving $v$ and Jacobian-vector product $\nabla_{xy}^2 g(x_k, \hy_k)\hv_k$, which can be efficiently computed and stored via existing automatic differentiation packages.
\subsection{ITD-based Bilevel Optimization Algorithm}
We present the single-loop ITD-based bilevel optimization algorithm (for simplicity, hereafter referred to as ITD) in Algorithm \ref{alg:bo-itd}. Similar to AID, ITD also performs one step of gradient descent and employs a warm-start strategy on the inner-level function $g(x,y)$ to obtain $\hy_k$. Unlike AID, however, ITD does not rely on the implicit gradient formula when estimating the hypergradient, but instead estimates the hypergradient directly via automatic differentiation. Since the update of $\hy_k$ depends on $x_k$, ITD needs to store the iterative trajectory for backpropagation. In this work, because we consider the more practical single-step gradient descent, the hypergradient estimate takes the following form: $\widehat\nabla \Phi(x_k)=\nabla_x  f(x_k, \hy_k)-\alpha\nabla_{xy}^2 g(x_k, y_k^0)\nabla_y f(x_k, \hy_k)$.

\section{Definitions and Assumptions}\label{sec:assum}
\begin{algorithm}[t]
	\caption{Single-Loop ITD-based bilevel optimization algorithm}   
%	\small
	\label{alg:bo-itd}
	\begin{algorithmic}[1]
		\STATE {\bfseries Input:}  Learning rate $\alpha,\beta>0$, initializations $x_0$ and $y_0$.
%	        \STATE {\bfseries} 
		\FOR{$k=0,1,2,...,K-1$}
		\STATE{Set $y_k^0 = \hy_{k-1} \mbox{ if }\; k> 0$ and $y_0$ otherwise \textbf{\em (warm start initialization)} }
%		\STATE{Set $v_k^0 = v_{k-1}^{N} \mbox{ if }\; k> 0$ and $v_0$ otherwise }
% y_0&  \mbox{if}\;k=0 \\
%y_{k-1}^{T} & \mbox{if}\; k\geq 0
%\end{array}\right. $
%		}
		% \FOR{$t=1,....,N$}
%		\STATE{Draw a sample batch $\gS_t$}  
		\vspace{0.05cm}
		\STATE{Update $\hy_k(x_k) = y_k^{0}-\alpha \nabla_y g(x_k,y_k^{0}) $}
		\vspace{0.05cm}
		% \ENDFOR
%                 \STATE{Draw sample batches $\gD_H,\gD_G,\gD_F$}
%                  \STATE{Construct Hessian inverse estimator $H\big(x_k,y_k^T;\gD_H\big)$ via \Cref{alg:hessianEst}}
                  \STATE{Compute hypergradient $\widehat\nabla \Phi(x_k)=\nabla_x  f(x_k, \hy_k)-\alpha\nabla_{xy}^2 g(x_k, y_k^0)\nabla_y f(x_k, \hy_k)$ 
                    }
                 \STATE{Update $x_{k+1}=x_k- \beta \widehat\nabla \Phi(x_k) $}
		\ENDFOR
	\end{algorithmic}
	\end{algorithm}
In~\eqref{eq:bo}, the objective is to minimize the hyper-objective function $\nabla\Phi(x)$, which is typically nonconvex. Because finding a global minimum for such functions can be computationally prohibitive \citep{nemirovskiĭ1983problem}, this work aims to find an approximate stationary point following the literature \citep{carmon2017lower, ji2021bilevel}.
\begin{definition}
    We call $\bar{x}$ is an $\epsilon$-stationary point of problem (\ref{eq:bo}) if $\norm{\nabla\Phi(\bar{x})}^2\leq\epsilon$.
\end{definition}
In this work, we focus on the problem (\ref{eq:bo}) under the following standard assumptions, as also widely adopted by \citet{ghadimi2018approximation, ji2021bilevel, dong2025efficient, fang2025qnbo}. Crucially, these standard assumptions align with \citet{ji2022will}, ensuring that the sharper results stem from our proposed technique, not other factors. Let $z = (x, y)$ denote all parameters.
\begin{assumption}
\label{assum-1}
The inner-level function $g(x, y)$ is $\mu$-strong-convex w.r.t. $y$.
\end{assumption}
\begin{assumption}\label{assum-2}
The function $f(z)$ is $M$-Lipschitz, i.e., for any $z,z'$,
\begin{align*}
\left|{f(z) - f(z')} \right|\leq M \norm{z - z'}.
\end{align*}
\end{assumption}
\begin{assumption}\label{assum-3}
	Gradients $\nabla f(z)$ and $\nabla g(z)$ are $L$-Lipschitz, i.e., for any $z,z'$,
	\begin{align*}
	&\norm{\nabla f(z) - \nabla f(z')} \leq L \norm{z - z'},\\
    &\norm{\nabla g(z) - \nabla g(z')} \leq L \norm{z - z'}.
	\end{align*}
\end{assumption}
\begin{assumption}\label{assum-4}
	Suppose the derivatives $\nabla_{xy}^2 g(z)$ and $\nabla_y^2 g(z)$ are $\rho$-Lipschitz, i.e., for any $z, z'$,
	\begin{align*}
		&\norm{ \nabla_{xy}^2 g(z) - \nabla_{xy}^2 g(z') } \leq \rho \norm{z - z'}, \\
        &\norm{\nabla_y^2 g(z) - \nabla_y^2 g(z')} \leq \rho \norm{z - z'}.
	\end{align*}
\end{assumption}

% \begin{assumption}\label{assum-4}
% 	Function $\nabla_y f( \cdot , y)$ for any given $y$ is $L$-Lipschitz continuous, that is, for $x, x'$, 
% 	\begin{equation}
% 		\norm{ \nabla_y f( x , y) - \nabla_y f( x' , y)} \leq L \norm{x - x'}.
% 	\end{equation}
% \end{assumption}
\paragraph{Condition-number convention.}
Throughout the displayed condition-number orders, we treat $L$, $M$, $\rho$, the initialization radii, and the initial objective gaps as $\gO(1)$ quantities independent of $\kappa=L/\mu$.
\section{Main Results}
In this section, we will provide the convergence analysis and characterize the overall computational complexity for both single-loop AID- and ITD-based algorithms.
\subsection{Challenges in the Analysis and Our Approach}
The conventional analytical path \citep{ji2021bilevel, ji2022will}, which we term \textit{Direct Squared Norm Analysis} (DSNA), relies on bounding the squared norm of the error vector at each iteration. Let's consider a simplified one-step error recurrence of the form $e_{k+1} = A e_k + \delta_k$, where $A$ represents the contraction operator and $\delta_k$ is the accumulated error term (e.g., from the inexact inner-loop solution). The standard approach proceeds by analyzing its squared norm:
$\norm{e_{k+1}}^2 = \norm{A e_k + \delta_k}^2 = \norm{A e_k}^2 + 2\langle A e_k, \delta_k\rangle + \norm{\delta_k}^2$.
The primary challenge arises from the cross-term, $2\langle A e_k, \delta_k\rangle$. To make this term tractable, existing analyses invariably resort to ``pessimistic" inequalities, such as the Cauchy-Schwarz or Young's inequality (e.g., $2\langle a, b\rangle \leq \norm{a}^2 + \norm{b}^2$). For example, \citet{ji2022will} adopted this approach when analyzing the error upper bounds of the inner variable and the solution of the linear system. While this decouples the terms, it does so at a great cost. This step fundamentally ignores any potential underlying structure or cancellation effects between $e_k$ and $\delta_k$. The repeated application of such loose bounds over many iterations causes the dependencies on the problem's condition number, $\kappa$, to compound, ultimately leading to the inflated convergence rate.
Our key insight is that this pessimistic rate is not an inherent property of the algorithm itself, but rather an analysis artifact stemming from the premature squaring of the norm. This step discards crucial information too early in the derivation.

We introduce a more delicate analytical strategy, \textit{Decoupled Norm Analysis} (DNA), that sidesteps this bottleneck. Instead of immediately squaring the error recurrence, we first analyze the error norm in its linear form by applying the triangle inequality:
$\norm{e_{k+1}} = \norm{A e_k + \delta_k} \leq\norm{A e_k} + \norm{\delta_k}$. By keeping the analysis in the linear domain of norms for as long as possible, we can establish a tighter recursive relationship (Lemmas \ref{lemma-3} and \ref{lemma-1} for AID, Lemma \ref{lemma:hy-ystar} for ITD). This approach allows for a more refined handling of the error terms, preserving more of the underlying geometric structure. The squaring operation is deferred to the very end of the analysis, after the full recurrence has been unrolled (Lemma \ref{lemma-6} for AID, Lemma \ref{lemma:hgPhi-gPhi-itd} for ITD). This seemingly simple change of order—analyzing the norm before squaring it—prevents the compounding of pessimistic estimates associated with the cross-term. It is this principled deviation from the standard analytical template that allows us to break the rate barrier and establish the significantly improved convergence rate, providing a more faithful theoretical picture of the algorithm's efficiency.
\subsection{Convergence Analysis of AID}
\textbf{Proof Sketch:} The proof for AID consists of three main steps: 1) Decomposing the hypergradient estimation error into the approximation error of the inner-level solution and the error from solving the linear system. (Lemma \ref{lemma-5}). 2) Bounding these two types of errors based on the errors in previous iterations (Lemmas \ref{lemma-3} and \ref{lemma-1}). 3) Combining the results from the preceding steps to provide a convergence guarantee for the AID algorithm (Theorem \ref{theo:aid}).

Before presenting the convergence analysis on AID, we first give the following useful lemmas. Now we study the convergence of  $\norm{\hv_k-v_k^*}$ and $\norm{\hy_k-y_k^*}$ for $k=1,2,\dots,K$, where $v_k^*$ is the exact solution of the linear system $\nabla_y^2 g(x_k,y_k^*) v =\nabla_y f(x_k,y_k^*)$. Note that the descent of the overall outer-level objectives also depends on the error of $y_k$. We next analyze these
errors.
\begin{lemma} \label{lemma-3}
    Consider single-loop AID-based algorithm in Algorithm \ref{alg:bo-aid}. Suppose Assumptions \ref{assum-1}-\ref{assum-4} hold. Let $\alpha\leq\frac{1}{L}$, then we have
 \begin{align}\label{eq:yy}
&\norm{y_k^0 - \hy_k}
\leq 
\alpha L \left(\norm{\hy_{k-1} - y_{k-1}^*}
+  \norm{x_{k-1} - x_k}  \right), \notag\\
&    \norm{\hy_k - y_k^*} 
\leq 
(1-\mu \alpha) \norm{\hy_{k-1} - y_{k-1}^*} + \frac{L}{\mu}\norm{x_{k-1} - x_k}.
\end{align}
\end{lemma}
\begin{remark}
    Lemma \ref{lemma-3} demonstrates that: 1) for $k=1,\dots, K$, the error between the initial point and the iterated solution of the inner-level problem in single-loop AID can be bounded by the error from the previous iteration; 2) the error between the approximate solution and the exact solution of the inner-level problem in single-loop AID can also be bounded by the error from the previous iteration, which serves as a crucial foundation for the analysis of the algorithm's convergence.
\end{remark}
Then, we decompose $\norm{\hv_k-v_k^*}$ and then estimate the upper bound.
\begin{lemma} \label{lemma-1}
Consider single-loop AID-based algorithm in Algorithm \ref{alg:bo-aid}. Suppose Assumptions \ref{assum-1}-\ref{assum-4} hold. Let $C_0=\frac{\rho M  }{\mu^2} +  \frac{L}{\mu}$. Let $\eta\leq\frac{1}{L}$, then we have
\begin{align}
\norm{\hv_k - v_k^*} 
\leq&  
\norm{ \hv_k - \tv_k^* }
+ C_0\norm{\hy_k - y_k^*}, \label{eq:vq-vstar}
\\
\norm{ \hv_k - \tv_k^* }
\leq &
(1-\mu \eta) \norm{ \hv_{k-1} - \tv_{k-1}^*}+ C_0  \left(\norm{y_k^0 - \hy_k}+\norm{x_{k-1} - x_k}\right), \label{eq:vq-tvstar}
\end{align}
where $\tv_k^*=(\nabla_y^2 g(x_k, \hy_k))^{-1}\nabla_y f(x_k, \hy_k)$.
\end{lemma}
\begin{remark}
    The purpose of Lemma \ref{lemma-1} is to conduct a more detailed decomposition of the error between $\hv_k$ and $v_k^*$, because this error originates from two aspects: 1) The use of $\hy_k$ to approximate $y_k^*$ in the inner-level problem. 2) The use of $\hv_k$, obtained from solving the linear system $\nabla_y^2 g(x_k,\hy_k) v =\nabla_y f(x_k,\hy_k)$, to approximate $v_k^*$. Therefore, Lemma \ref{lemma-1} decouples these two factors and controls them separately. Specifically, the first and second terms in \eqref{eq:vq-vstar} are only related to the precision of the linear equation solution and the inner-level problem solution, respectively. \eqref{eq:vq-tvstar} further expands the first term on the right-hand side of \eqref{eq:vq-vstar}.
\end{remark}
% Then we aim to bound the first term in the right side of \eqref{eq:vq-vstar} via the error of the former iteration and the other variables. 
% \begin{lemma} \label{lemma-2}
% Consider single-loop AID-based algorithm in Algorithm \ref{alg:bo-aid}. Suppose Assumptions \ref{assum-1}-\ref{assum-4} hold. Then, we have
% \begin{align}\label{eq:vq-tvstar}
% \norm{ \hv_k - \tv_k^* }
% \leq 
% (1-\mu \eta) \norm{ \hv_{k-1} - \tv_{k-1}^*}
% + C_0  \left(\norm{y_k^0 - \hy_k}+\norm{x_{k-1} - x_k}\right),
% \end{align}
% where $C_0$, $v_k^*$, $\tv_k^*$ and $y_k^*$ is defined in Lemma \ref{lemma-1}.
% \end{lemma}

% \begin{lemma}\label{lemma-4}
%     Suppose Assumptions \ref{assum-1}, \ref{assum-2}, \ref{assum-3} and \ref{assum-4} hold. Let $\alpha\leq\frac{1}{L}$, then we have
% \begin{equation}
% \norm{\hy_k - y_k^*} 
% \leq 
% (1-\mu \alpha) \norm{\hy_{k-1} - y_{k-1}^*} + \frac{L}{\mu}\norm{x_{k-1} - x_k}.
% \end{equation}
% \end{lemma}

In Lemmas \ref{lemma-3} and \ref{lemma-1}, we have already provided the relevant error terms of $y_k$ and $v_k$. Therefore, we will utilize the above results to analyze the error between the estimated hypergradient $\widehat{\nabla}\Phi(x_k)$ and the true hypergradient ${\nabla}\Phi(x_k)$.
\begin{lemma}\label{lemma-5}
Consider single-loop AID-based algorithm in Algorithm \ref{alg:bo-aid}. Suppose Assumptions \ref{assum-1}-\ref{assum-4} hold. Define $C_0$ as in Lemma \ref{lemma-1}. Then we have
\begin{align}\label{eq:nn}
\norm{\widehat{\nabla} \Phi(x_k) - \nabla \Phi(x_k)}
\leq &
\left(L + \frac{\rho M}{\mu} + C_0 L\right) \norm{\hy_k -y_k^*} + L \norm{\hv_k - \tv_k^*}.
\end{align}
\end{lemma}
Unlike the previous DSNA, our proposed DNA avoids the inflation of the condition number $\kappa$ caused by repeated squaring. Combine \eqref{eq:nn} with the former lemmas, we can get the following lemma.
\begin{lemma}\label{lemma-6}
Consider single-loop AID-based algorithm in Algorithm \ref{alg:bo-aid}. Suppose Assumptions \ref{assum-1}-\ref{assum-4} hold. Define $C_0$ as in Lemma \ref{lemma-1}. Let $\alpha=\eta=\frac{1}{L}$, $C_1=\frac{4C_0L}{\mu}$, $C_2={\frac{\alpha L^2 C_0}{\mu} + \frac{\rho M}{\mu^2}+\frac{L}{\mu} + \frac{LC_1}{\mu} }$ and $C_3=L + \frac{\rho M}{\mu} + C_0 L$. Choose the outer stepsize $\beta$ such that $\beta\leq\min\{\frac{C_1\mu\alpha}{4C_2C_3}, \frac{\eta\mu}{2LC_2}\}$. Let $E_0=\norm{\hv_0-\tv_0^*}+C_1\norm{\hy_0-y_0^*}$. Then, we have
\begin{align}\label{eq:pp}
\begin{aligned}
\norm{\widehat{\nabla}\Phi(x_k) - \nabla \Phi(x_k)}^2\leq
L^2 \left(1-\frac{\mu}{4L}\right)^k \cdot E_0^2+
\frac{3\beta^2 C_2^2 L^3}{\mu}\sum_{t=0}^{k-1} \left(1-\frac{\mu}{4L}\right)^{k-1-t} \norm{\nabla\Phi(x_t)}^2.
\end{aligned}
\end{align}
\end{lemma}
\begin{remark}
Lemma \ref{lemma-6} is a key result that supports the convergence analysis of single-loop AID-based algorithm. Compared to the work of \citet{ji2022will}, we relax the limit of the step-size for solving the linear system. Specifically, \citet{ji2022will} in their Corollary 2 required that $\eta = \gO(\kappa^{-2})$, whereas we, through a more fine-grained analysis, set eta to $1/L$. This indirectly allows for a more aggressive choice of the outer-level step size $\beta$, thereby achieving a faster convergence rate.
\end{remark}

\begin{lemma}[Initialization bound for AID]\label{lem:aid-init}
Consider the single-loop AID-based algorithm in Algorithm~\ref{alg:bo-aid}. Suppose Assumptions~\ref{assum-1}-\ref{assum-4} hold, $\alpha=\eta=1/L$, and $v_0=0$. Let $y_{0,\mathrm{raw}}$ be any initial inner point with $D_y=\norm{y_{0,\mathrm{raw}}-y_0^*}<\infty$. Before running Algorithm~\ref{alg:bo-aid}, perform $N_0$ gradient descent warm-up steps on $g(x_0,\cdot)$ with stepsize $1/L$ and set $y_0$ to the warm-start output, where
\begin{align*}
N_0\geq 
\left\lceil
\frac{\log\left(D_y/(c_y\mu)\right)_+}{\log\left(1/(1-\mu/L)\right)}
\right\rceil
\end{align*}
for an absolute constant $c_y>0$, and $\log(a)_+=\max\{\log(a),0\}$. Then, for $E_0=\norm{\hv_0-\tv_0^*}+C_1\norm{\hy_0-y_0^*}$ defined in Lemma~\ref{lemma-6}, we have
\begin{align*}
\frac{L^3}{\mu}E_0^2=\gO(\kappa^5).
\end{align*}
\end{lemma}

Based on the above conclusions, the following theorem provides a convergence analysis for single-loop AID-based algorithm. We use the initial-error quantity $E_0=\norm{\hv_0-\tv_0^*}+C_1\norm{\hy_0-y_0^*}$ from Lemmas~\ref{lemma-6} and~\ref{lem:aid-init}.
% \begin{tcolorbox}
\begin{theorem}\label{theo:aid}
Consider the single-loop AID-based algorithm in Algorithm \ref{alg:bo-aid}. Suppose Assumptions \ref{assum-1}-\ref{assum-4} hold and $\Phi^\star:=\inf_x\Phi(x)>-\infty$. Initialize $v_0=0$, choose $y_0$ according to Lemma~\ref{lem:aid-init}, and set $\alpha=\eta=1/L$. The smoothness parameter of $\Phi$ is
\begin{align*}
L_\Phi={}&L+\frac{2L^2+\rho M^2}{\mu}
+\frac{2\rho LM+L^3}{\mu^2}
+\frac{\rho L^2M}{\mu^3}.
\end{align*}
For the origin of $L_\Phi$, see Lemma \ref{le:lipphi}. Choose a fixed numerical constant $c_{\rm aid}\in(0,1]$ independent of $\kappa$ and set
\begin{align*}
\beta=c_{\rm aid}\min\left\{\frac{1}{8L_\Phi},\frac{C_1\mu\alpha}{4C_2C_3}, \frac{\eta\mu}{2LC_2},\frac{\mu}{8C_2L^2}\right\}.
\end{align*}
Then $\beta=\Theta(\kappa^{-5})$ under the condition-number bookkeeping convention above, and
\begin{align*}
\frac{1}{K}\sum_{k=0}^{K-1}\norm{\nabla\Phi(x_k)}^2=\gO\left(\frac{\kappa^5}{K}\right),\quad
\mathrm{Gc}(\epsilon)=\widetilde{\gO}(\kappa^5\epsilon^{-1}),
\quad\mathrm{Mv}(\epsilon)=\widetilde{\gO}(\kappa^5\epsilon^{-1}).
\end{align*}
\end{theorem}
% \end{tcolorbox}

{\remark 
Compared with the work of \citet{ji2022will}, our core improvement lies in controlling the errors of both the inner solution $y$ and the linear system solution $v$, where we relax the requirement on the outer objective learning rate $\beta$ from $\gO(\kappa^{-6})$ to $\gO(\kappa^{-5})$. Consequently, we improve the convergence rate of single-loop AID-based algorithm from $\gO(\kappa^{6}/K)$ to $\gO(\kappa^{5}/K)$. This indicates that the convergence gap between such algorithms and the AID algorithms with multi-step gradient descent is not as large as the $\gO(\kappa^{2})$ gap shown by \citet{ji2022will}, but rather a smaller $\gO(\kappa^{1})$. This also partially supports the practice that most bilevel optimization algorithms perform only one or a few inner updates.
}
\begin{theorem} \label{theo:aid-ji}
{\rm\textbf{[Simplified version of the upper bound in \citet{ji2022will}]}}. Consider single-loop AID-based algorithm in Algorithm \ref{alg:bo-aid}. Under the same setting of Theorem \ref{theo:aid}, we have $\frac{1}{K}\sum_{k=0}^{K-1}\|\nabla\Phi(x_k)\|^2 =\mathcal{O}\big( \frac{\kappa^6}{K}\big)$.
\end{theorem}
\subsection{Convergence Analysis of ITD}
\textbf{Proof Sketch:} Unlike AID, the hypergradient estimation error of the single-loop ITD-based algorithm is introduced only by solving the inner problem. Therefore, our proof consists of three main steps: 1) Establishing the connection between the hypergradient estimation error and the approximation error of the inner-level solution (Lemma \ref{lem:hPhi-Phi}). 2) Bounding the approximation error of the solution to the inner-level problem (Lemma \ref{lemma:hy-ystar}). 3) Combining the results from the previous steps to provide a convergence analysis for the ITD algorithm (Lemma \ref{lemma:hgPhi-gPhi-itd} and Theorem \ref{theo:itd}).

To this end, we first present several useful lemmas, which will subsequently be used to prove Theorem \ref{theo:itd}.
\begin{lemma}[Radius of ITD warm-start iterates]\label{lem:itd-radius}
Consider the single-loop ITD-based algorithm in Algorithm~\ref{alg:bo-itd}. Suppose Assumptions~\ref{assum-1}-\ref{assum-4} hold and let $\alpha\leq1/L$. Define $a_k=\norm{y_k^0-y_k^*}$ and $G=M(1+\alpha L)$. If $a_0\leq R_y$ and
\begin{align*}
\beta\leq \frac{\alpha\mu^2 R_y}{LG},
\end{align*}
then $a_k\leq R_y$ for all $k\geq0$. In particular, for $\alpha=1/L$ and $\beta\leq\frac{\mu^3}{2L(2L^2+\rho M)}$, the above condition is satisfied by any constant $R_y$ such that
\begin{align*}
R_y\geq \frac{\mu L M}{2L^2+\rho M}.
\end{align*}
Moreover, if the raw initialization satisfies $D_y=\norm{y_{0,\mathrm{raw}}-y_0^*}<\infty$, then $a_0\leq R_y$ can be enforced by finitely many gradient-descent warm-up steps at $x_0$. Hence, under the paper's condition-number scaling, one may choose $R_y=\gO(1)$ and obtain $\norm{y_k^0-y_k^*}\leq R_y$ for all $k$.
\end{lemma}

We first record the ITD hypergradient-estimator error, since this bound is used in the recursion for the inner warm-start error.
\begin{lemma}\label{lem:hPhi-Phi}
Consider the single-loop ITD-based algorithm in Algorithm \ref{alg:bo-itd}. Suppose Assumptions \ref{assum-1}-\ref{assum-4} hold and the conditions of Lemma~\ref{lem:itd-radius} hold. Let $\alpha\leq\frac{1}{L}$ and define
\begin{align*}
C_4=L+\alpha L^2+\alpha\rho M,\quad C_5=M(1-\alpha\mu)\frac{L}{\mu}+\alpha\rho M R_y.
\end{align*}
Then
\begin{align}
    \norm{\widehat{\nabla}\Phi(x_k)-\nabla\Phi(x_k)}
    &\leq C_4\norm{\hy_k-y^*_k}+C_5.
\end{align}
% where $C_4=L+\alpha L^2 + \alpha\rho M$ and $C_5=M(1-\alpha\mu)\frac{L}{\mu} + \alpha^2\rho M^2$.
\end{lemma}
\begin{remark}
    Lemma \ref{lem:hPhi-Phi} shows that the error between the true hypergradient and the estimated hypergradient is controlled by the accuracy of the inner-level problem solution and an inherent error, part of which arises from $\norm{y_k^0-\hy_k}$. This indicates that this non-vanishing convergence error is related to the refinement of the inner-level problem solution, and that the single-loop method is insufficient to bridge this gap.
\end{remark}

\begin{lemma}\label{lemma:hy-ystar}
Consider the single-loop ITD-based algorithm in Algorithm \ref{alg:bo-itd}. Suppose Assumptions \ref{assum-1}-\ref{assum-4} hold and the conditions of Lemma~\ref{lem:itd-radius} hold. Let
\begin{align*}
\alpha=\frac{1}{L},\quad\beta\leq\min\left\{
\frac{\mu^3}{2L(2L^2+\rho M)},
\frac{\mu^2}{16L^2C_4}
\right\}.
\end{align*}
Define $C_4,C_5$ as in Lemma~\ref{lem:hPhi-Phi} and set
\begin{align*}
C_6=1-\mu\alpha+\frac{L\beta C_4}{\mu},
\quad C_7=\frac{L\beta C_5}{\mu}.
\end{align*}
Then we have
\begin{align}
    \norm{\hy_k-y^*(x_k)}\leq &C_6\norm{\hy_{k-1}-y^*(x_{k-1})} + \frac{L\beta}{\mu}\norm{{\nabla}\Phi (x_{k-1})}+C_7,\label{eq:itd:yhat-ystar}\\
    \norm{\hy_k-y^*(x_k)}\leq& \left(1-\frac{\mu}{2L}\right)^k\norm{\hy_0-y^*(x_0)} + \frac{L\beta}{\mu}\sum_{j=0}^{k-1}
    \left(1-\frac{\mu}{2L}\right)^{k-1-j}\left(\norm{\nabla\Phi(x_j)}+C_5\right).
    \label{eq:itd:yhat-ystar-v2}
\end{align}
\end{lemma}
\begin{lemma}\label{lemma:hgPhi-gPhi-itd}
Consider the single-loop ITD-based algorithm in Algorithm \ref{alg:bo-itd}. Suppose Assumptions \ref{assum-1}-\ref{assum-4} hold and the conditions of Lemma~\ref{lem:itd-radius} hold. Define $C_4$ and $C_5$ in Lemma \ref{lem:hPhi-Phi}. Let $\alpha=1/L$ and assume $\beta\leq\mu^3/[2L(2L^2+\rho M)]$. In addition, assume $\beta\leq\mu^2/(16L^2C_4)$. Define
\begin{align*}
\Delta_k^{\rm itd}:=\widehat{\nabla}\Phi(x_k)-\nabla\Phi(x_k).
\end{align*}
Then we have
\begin{small}
\begin{align*}
    \norm{\Delta_k^{\rm itd}}^2
    &\leq 2C_4^2\left(1-\frac{\mu}{4L}\right)^k\norm{\hy_0-y^*(x_0)}^2 \\
    &\quad+\frac{6L^3\beta^2C_4^2}{\mu^3}
    \sum_{j=0}^{k-1}\left(1-\frac{\mu}{4L}\right)^{k-1-j}\left(\norm{\nabla\Phi(x_j)}+C_5\right)^2+2C_5^2.
\end{align*}
\end{small}
\end{lemma}Based on the above results, the following theorem provides a convergence analysis for single-loop ITD-based algorithm.
% \begin{tcolorbox}
\begin{theorem} \label{theo:itd}
Consider the single-loop ITD-based algorithm in Algorithm \ref{alg:bo-itd}. Suppose Assumptions \ref{assum-1}-\ref{assum-4} hold, $\Phi^\star:=\inf_x\Phi(x)>-\infty$, and initialize $y_0$ as in Lemma~\ref{lem:itd-radius} so that $\norm{y_0-y_0^*}\leq R_y$, where $R_y=\gO(1)$ satisfies $R_y\geq\frac{\mu L M}{2L^2+\rho M}$. Choose $\alpha=1/L$, and let the smoothness parameter of $\Phi$ be
\begin{align*}
L_\Phi={}&L+\frac{2L^2+\rho M^2}{\mu}
+\frac{2\rho LM+L^3}{\mu^2}
+\frac{\rho L^2M}{\mu^3}.
\end{align*}
Let $C_4$ be defined as in Lemma~\ref{lem:hPhi-Phi} with $\alpha=1/L$. Choose a fixed numerical constant $c_{\rm itd}\in(0,1]$ independent of $\kappa$ and set
\begin{align*}
\beta=c_{\rm itd}\min\bigg\{
\frac{\mu^3}{2L(2L^2+\rho M)},
\frac{1}{8L_\Phi},\frac{\mu^2}{16L^2C_4}\bigg\}.
\end{align*}
Then $\beta=\Theta(\kappa^{-3})$ under the condition-number bookkeeping convention above, and
\begin{align*}
\frac{1}{K}\sum_{k=0}^{K-1}\norm{\nabla\Phi(x_k)}^2
=\gO\left(\frac{\kappa^3}{K}+\kappa^2\right).
\end{align*}
\end{theorem}
% \end{tcolorbox}
\begin{remark}
    Theorem \ref{theo:itd} demonstrates that for the single-loop ITD-based algorithm, the convergence bound contains a non-vanishing error of order $\gO(\kappa^2)$. Under the standard Assumptions \ref{assum-1}-\ref{assum-4}, such an error is unavoidable. Moreover, this error upper bound of order $\gO(\kappa^2)$ matches the error lower bound (Theorem \ref{theo:itd-lower}), which indicates that we have achieved a tighter error upper bound through more refined analysis. This resolves the issue in \citet{ji2022will} where there exists a gap of $\alpha\mu$ between the upper and lower bounds.
\end{remark}
\begin{theorem}\label{theo:itd-lower}
{\rm\textbf{[Worst-case lower bound adapted from \citet{ji2022will}]}} There exists a problem instance satisfying Assumptions \ref{assum-1}-\ref{assum-4} such that, for the single-loop ITD-based algorithm in Algorithm \ref{alg:bo-itd} with $\alpha\leq\frac{1}{L}$ and $\beta\leq\frac{1}{L_\Phi}$, the worst-case non-vanishing stationarity error satisfies $\liminf_{K\to\infty}\|\nabla\Phi(x_K)\|^2 \geq c\kappa^2$ for a universal constant $c>0$.
\end{theorem}
% {\remark}
\section{Experiments}\label{sec:exp-main}
\subsection{Numerical Verification}
\begin{figure}[!t]
\vspace{0.1cm}
    \centering
    \includegraphics[scale=0.49]{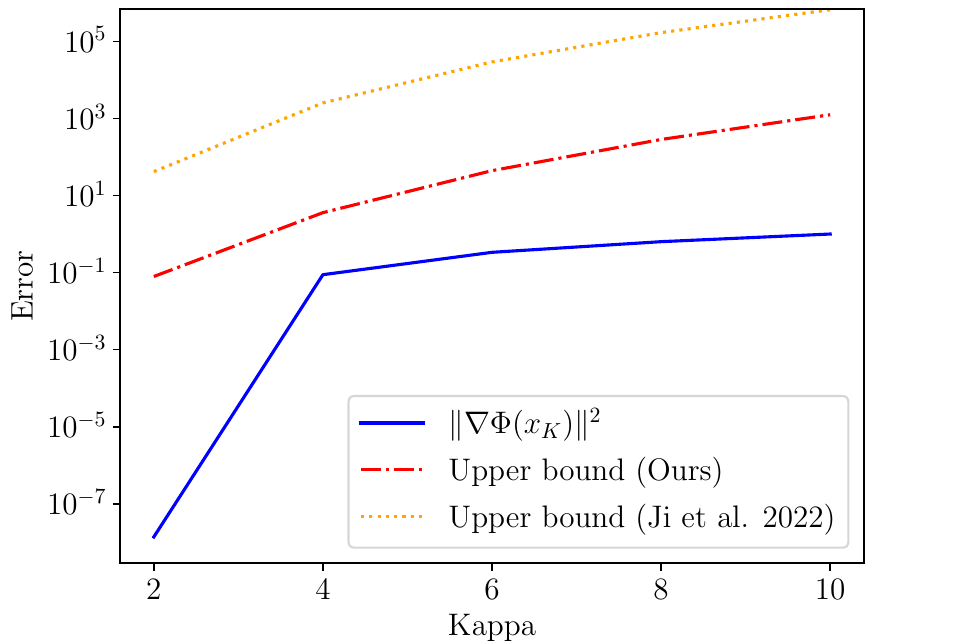}
    \hspace{-0.25cm}
    \includegraphics[scale=0.49]{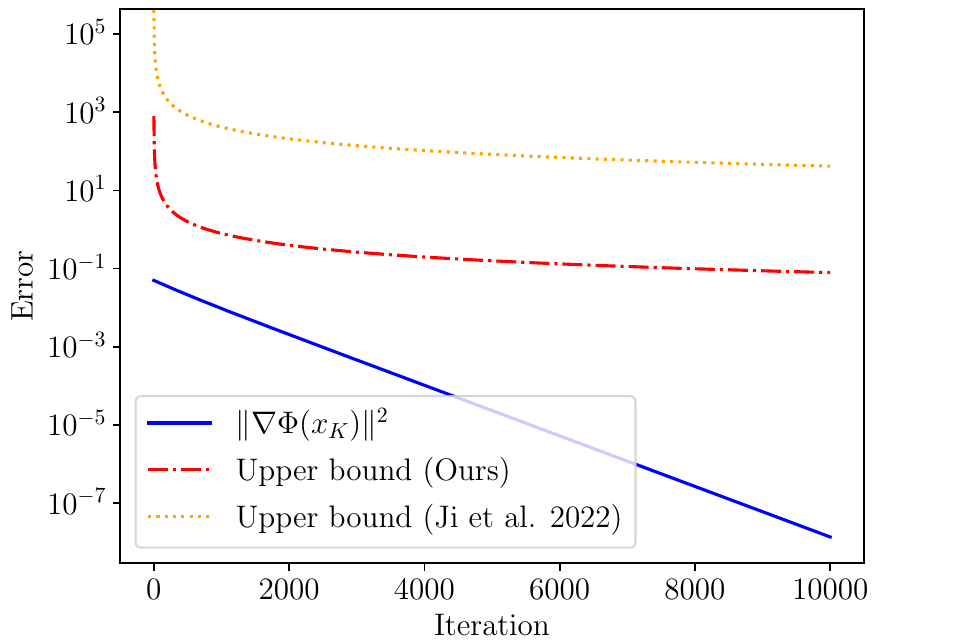}
    \caption{Comparison of error curves of the single-loop AID-based Algorithm. \textbf{Left:} Curves of various error terms (the squared norm of the true hypergradient $\norm{\nabla\Phi(x_k)}^2$, the upper bound provided in Theorem \ref{theo:aid} by us, and the upper bound provided in Theorem \ref{theo:aid-ji} by \cite{ji2022will}) with respect to different condition numbers $\kappa$. \textbf{Right:} Curves of various error terms with respect to the number of iterations $K$ when the condition number $\kappa=2$.}
    \label{fig:aid}
\vspace{-0.27cm}
\end{figure}

\begin{figure}[!t]
\vspace{0.1cm}
    \centering
    \includegraphics[scale=0.49]{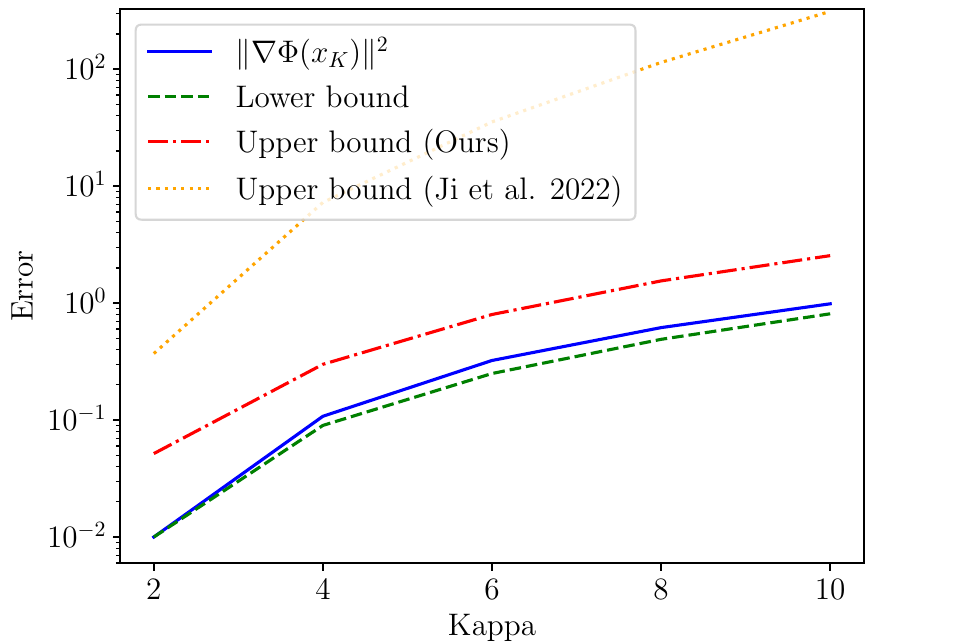}
    \hspace{-0.25cm}
    \includegraphics[scale=0.49]{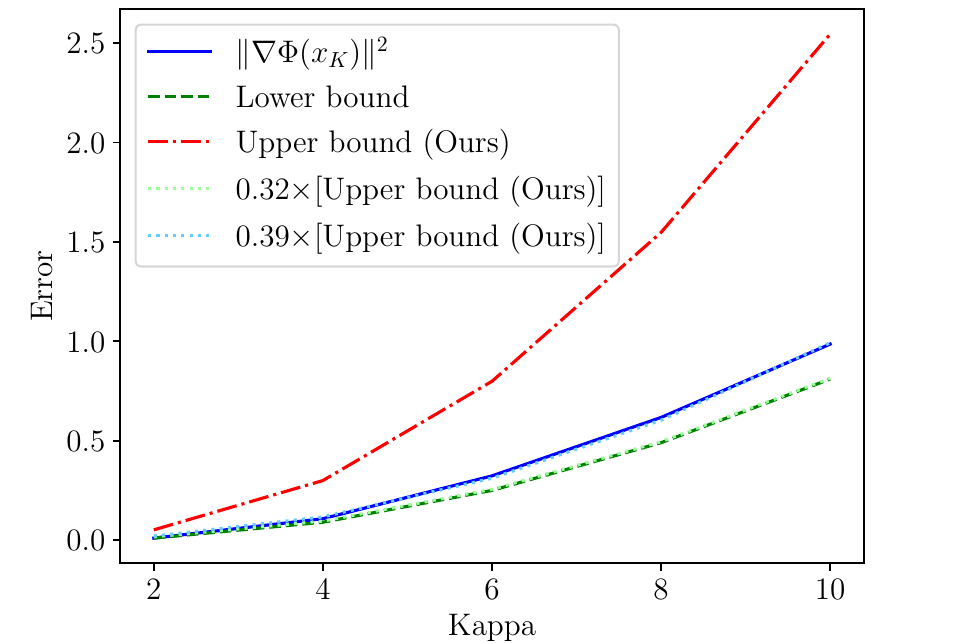}
    \caption{Comparison of error curves of the single-loop ITD-based Algorithm. \textbf{Left:} Curves of various error terms (the squared norm of the true hypergradient $\norm{\nabla\Phi(x_k)}^2$, the upper bound provided in Theorem \ref{theo:itd} by us, and the upper bound provided by \cite{ji2022will}, the lower bound provided in Theorem \ref{theo:itd-lower}) with respect to different condition numbers $\kappa$. \textbf{Right:} Curves of the scaled upper bound ($\times0.32$ and $\times0.39$) with respect to different condition numbers $\kappa$.}
    \label{fig:itd-kappa}
\vspace{-0.27cm}
\end{figure}
\textbf{Experimental setup.} We consider the following bilevel optimization problem:
\begin{align*}\label{worst_case_instance}
&f(x,y) = \frac{1}{2}x^TZ_xx + \frac{1}{10}\mathbf{1}^Ty,\\
&g(x,y)= \frac{1}{2} y^TZ_yy - L x^Ty + \mathbf{1}^Ty,
\end{align*}
where $x$, $y\in\sR^2$ and $Z_x = Z_y = \begin{bmatrix}
 L&0  \\
 0  &  \mu  \\ 
\end{bmatrix}$. Thus the optimal solution of the inner-level subproblem and the exact hypergradient have the following form:
\begin{align*}
y^*(x)  = Z_y^{-1}(Lx-\mathbf{1}), \quad \nabla\Phi(x) & = Z_xx + L Z_y^{-1} \mathbf{1}.
\end{align*}
Based on the updates of single-loop method, we have $\hy_k=y_k^0-\alpha (Z_yy_{k}^{0}-Lx_k+\mathbf{1})$. Let the hyperparameters set as $\mu=0.1$, $M=0.1$, $\rho=0.1$, $K=10000$ and $\alpha=1/L$.

\begin{figure}[!t]
\vspace{-0.25cm}
    \centering
    \includegraphics[scale=0.44]{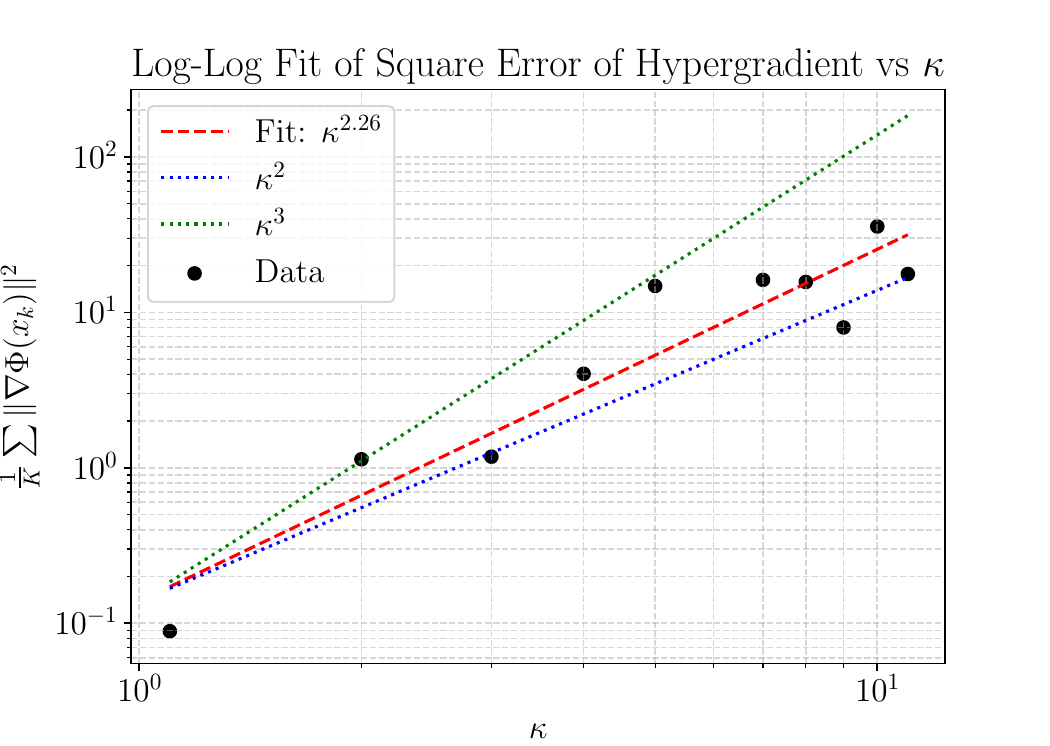}
    \hspace{-0.25cm}
    \includegraphics[scale=0.44]{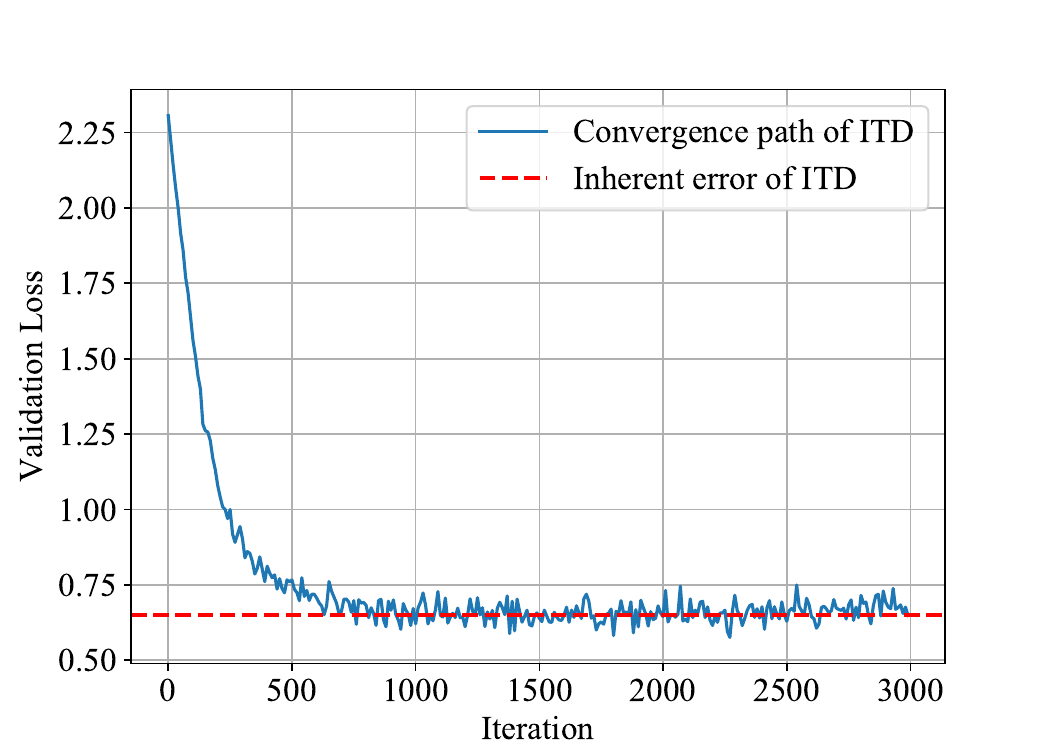}
    \caption{\textbf{Left:} Comparison of error curves of the single-loop ITD-based Algorithm on the task of feature learning \citep{bao2021stability} of the space dataset \citep{CC01a}. Curves of the upper bound of $\frac{1}{K}\sum\norm{\nabla\Phi(x_k)}^2$ with respect to different condition numbers $\kappa$. \textbf{Right:} Verification of the inherent error of the single-loop ITD-based Algorithm. We conduct data reweighting experiment on the MNIST dataset \cite{lecun1998gradient} with 20\% label corruption. The validation loss ultimately exhibits an inherent error of approximately 0.65.}
    \label{fig:itd-kappa-real}
\vspace{-0.27cm}
\end{figure}

\textbf{Results of AID-based Algorithm.} Figure \ref{fig:aid} presents the error curves of the single-loop AID-based method. In Figure \ref{fig:aid} (Left), we compare the error upper bound derived by Theorem \ref{theo:aid} with that given by \citet{ji2022will} under different condition numbers $\kappa$. It can be observed that, under varying condition numbers, our upper bound curve consistently lies closer above the $\norm{\nabla\Phi(x_k)}^2$ curve. This is achieved by refining the analysis and reducing the theoretical order of the upper bound from $\gO(\kappa^6)$ to $\gO(\kappa^5)$. In Figure \ref{fig:aid} (Right), under the condition number $\kappa = 2$, we compare the variation of the error upper bound with respect to the number of outer iterations $K$. It can be seen that the $\norm{\nabla\Phi(x_k)}^2$ curve keeps decreasing as the number of iterations increases, which indicates that the single-loop AID-based algorithm converges as $K$ grows, thereby confirming the correctness of Theorem \ref{theo:aid}. Moreover, we observe that our upper bound curve consistently outperforms that of \citet{ji2022will}, which demonstrates that, theoretically, we provide a tighter error upper bound for this algorithm, thus verifying the correctness and effectiveness of our theoretical results.

\textbf{Results of ITD-based Algorithm.} Figure \ref{fig:itd-kappa} illustrates the performance of the ITD-based algorithm. From Figure \ref{fig:itd-kappa} (Left), we first observe that in \citet{ji2022will}, the gap between the reported upper and lower bounds remains large, confirming their conclusion that both bounds still differ by an error of order $\alpha\mu$. In contrast, our theoretical upper bound is substantially tighter: it lies much closer to the empirical $\norm{\nabla\Phi(x_K)}^2$ curve while remaining strictly above it. This demonstrates that our bound provides a sharper characterization of the true convergence behavior.

To further verify the validity of our theoretical results, in addition to the curve of the true hypergradient norm, the upper bound curve (according to Theorem \ref{theo:itd}), and the lower bound curve, we also scale the upper bound curve in Figure \ref{fig:itd-kappa} (Right). Specifically, we multiply it by 0.32 and 0.39, respectively. The results show that, after scaling the upper bound curve with different factors, its error values almost coincide with the true hypergradient norm curve and the lower bound curve, respectively. This indicates that the difference between the upper bound and the lower bound arises from constant factors introduced by scaling, rather than from differences in order. Thus, this supports the conclusion of Theorem \ref{theo:itd}, namely that we have reduced the inherent error to $\gO(\kappa^2)$.

\subsection{Verification on Real Data}
We further conduct validation on the space dataset \cite{CC01a} and MNIST dataset \cite{lecun1998gradient} for the bilevel optimization task of feature learning \cite{bao2021stability} and data reweighting \cite{shaban2019truncated}. The experimental settings are provided in Appendix \ref{sec:exp-set}.

\textbf{Results of feature learning.} To evaluate the effectiveness of our theoretical findings, we carry out experiments focusing on the ITD-based algorithm with the mentioned feature learning task. By modifying the singular values of the feature matrix in the training set, we were able to directly control the condition number $\kappa$. The final experimental results are shown in Figure~\ref{fig:itd-kappa-real} \textbf{(Left)}. For different values of the condition number $\kappa$, we record the mean squared norm of the hypergradient $\frac{1}{K}\sum_{k=0}^{K-1}\norm{\nabla\Phi(x_k)}^2$ (which can be computed exactly using the closed-form solution of the head layer). The results demonstrate that the empirical dependence between the hypergradient’s mean squared norm and the condition number follows $\gO(\kappa^{2.26})$. This dependence is closer to $\gO(\kappa^{2})$ than to $\gO(\kappa^{3})$, thereby supporting the soundness of our theoretical conclusions.

Moreover, the observed exponent $\gO(\kappa^{2.26})$ is likely due to limitations on the number of iterations, which prevent the term 
$\gO(\kappa^{3}/K)$ from becoming negligible. As a result, the final empirical rate appears as $\gO(\kappa^{2.26})$, lying closer to $\gO(\kappa^{2})$, consistent with our theoretical predictions.

\textbf{Results of data reweighting.} Figure \ref{fig:itd-kappa-real} \textbf{(Right)} illustrates the convergence path of the validation loss for single-loop ITD-based algorithms in the data-reweighting task. As the number of iterations increases, the outer-level validation loss decreases rapidly and then stays near a fixed value (approximately 0.65 in this experiment). This observation indicates that single-loop ITD methods indeed exhibit an inherent, non-vanishing error, which supports the soundness of our theoretical results.
\section{Conclusion}
In this work, we advance the theoretical understanding of single-loop bilevel optimization algorithms, a setting of growing practical relevance. For the AID method, our refined analysis improves the convergence rate to $\mathcal{O}(\kappa^5/K)$, narrowing the gap with the practical performance. For the ITD method, we establish that its convergence error is exactly $\mathcal{O}(\kappa^2)$, thereby closing the open question raised in prior work regarding its tightness. Our experimental results can corroborate the theory, demonstrating that single-loop methods can achieve both efficiency and favorable convergence behavior. These findings not only bridge an important gap between theory and practice, but also potentially suggest that the single-loop bilevel optimization methods can be strong candidates for large-scale machine learning tasks. Beyond the specific result for the algorithm, we believe our proposed analytical paradigm of the decoupling norm analysis opens new path for studying other bilevel optimization algorithms, potentially tightening bounds for methods where previous analyses have been overly pessimistic.

\bibliography{example_paper}
\bibliographystyle{plainnat}

\newpage
\appendix
 \section{Proof of the single-loop AID-based algorithm}
{Firstly, we give the following useful lemma. Recall that  $\Phi(x)=f(x,y^*(x))$ in~\eqref{eq:bo}. Then, we use 
the following lemma to characterize the Lipschitz properties of $\nabla \Phi(x)$, which is adapted from Lemma 2.2 in~\citealt{ghadimi2018approximation}.
\begin{lemma}\label{le:lipphi}
Suppose Assumptions~\ref{assum-1}-~\ref{assum-4} hold. Then, we have, for any $x,x^\prime\in\mathbb{R}^p$,  
\begin{align*}
\|\nabla \Phi(x)- \nabla \Phi(x^\prime)\| \leq L_\Phi \|x-x^\prime\|,
\end{align*}
where the constant $L_\Phi$ is given by
\begin{align}
L_\Phi = L + \frac{2L^2+\rho M^2}{\mu} + \frac{2\rho L M+L^3}{\mu^2} + \frac{\rho L^2 M}{\mu^3}.%=\mathcal{O}(1+\frac{}{}).
\end{align}
\end{lemma}}
\subsection{Proof of Lemma \ref{lemma-3}}
\begin{proof}
By the update rule of $y_k$, we have for each $k=1,\dots,K,$
\begin{align*}
\norm{y_k^0 - \hy_k}
=& 
\alpha \norm{\nabla_y g(x_k, y_k^0)}
=
\alpha \norm{\nabla_y g(x_k, \hy_{k-1})}\\
=& 
\alpha \norm{\nabla_y g(x_k, \hy_{k-1}) - \nabla_y g(x_k, y_{k-1}^*)
	+ \nabla_y g(x_k, y_{k-1}^*) -\nabla _y g(x_{k-1}, y_{k-1}^*)}\\
\leq&
\alpha L \left(\norm{\hy_{k-1} - y_{k-1}^*}
+  \norm{x_{k-1} - x_k}  \right).
\end{align*}
The second conclusion holds that
\begin{align*}
\norm{\hy_k - y_k^*} 
\leq& 
(1-\mu \alpha) \norm{y_k^0 - y_k^*} 
\leq 
(1-\mu \alpha) \norm{\hy_{k-1} - y_{k-1}^*} + \norm{y_{k-1}^* - y_k^*}\\
\stackrel{(\text{i})}{\leq}&
(1-\mu \alpha) \norm{\hy_{k-1} - y_{k-1}^*} + \frac{L}{\mu}\norm{x_{k-1} - x_k},
\end{align*}
where (i) follows from Lemma 2.2 in \citet{ghadimi2018approximation}.
\end{proof}
\subsection{Proof of Lemma \ref{lemma-1}}
In the following two proofs, we will respectively present the two conclusions (\eqref{eq:vq-vstar} and \eqref{eq:vq-tvstar}) in Lemma \ref{lemma-1}.
\begin{proof}
    According to the triangle inequality, we have $\norm{\hv_k - v_k^*} 
\leq 
\norm{ \hv_k - \tv_k^* }+\norm{ \tv_k^* - v_k^* }$ for $k=1, 2, \dots, K$. Then we focus on using $\norm{\hy_k - y_k^*}$ to bound $\norm{ \tv_k^* - v_k^* }$:
\begin{align*}
\norm{ \tv_k^* - v_k^*  }
=&
\norm{ [\nabla_y^2 g(x_k, \hy_k)]^{-1} \nabla_y f(x_k, \hy_k) - [\nabla_y^2 g(x_k, y_k^*)]^{-1} \nabla_y f(x_k, y_k^*) }\\
\leq&
\norm{ [\nabla_y^2 g(x_k, \hy_k)]^{-1} \nabla_y f(x_k, \hy_k) - [\nabla_y^2 g(x_k, y_k^*)]^{-1}\nabla_y f(x_k, \hy_k) } \\
&+
\norm{[\nabla_y^2 g(x_k, y_k^*)]^{-1}\nabla_y f(x_k, \hy_k) - [\nabla_y^2 g(x_k, y_k^*)]^{-1} \nabla_y f(x_k, y_k^*) }\\
\leq&
\norm{[\nabla_y^2 g(x_k, \hy_k)]^{-1} - [\nabla_y^2 g(x_k, y_k^*)]^{-1}} \cdot \norm{ \nabla_y f(x_k, \hy_k)}\\
&
+ \norm{[\nabla_y^2 g(x_k, y_k^*)]^{-1}} \cdot \norm{\nabla_y f(x_k, \hy_k) - \nabla_y f(x_k, y_k^*)}\\
\leq&
\frac{\rho M \norm{\hy_k - y_k^*} }{\mu^2} + \frac{L}{\mu} \norm{\hy_k - y_k^*} 
=
\left( \frac{\rho M  }{\mu^2} +  \frac{L}{\mu}\right)\norm{\hy_k - y_k^*}.
\end{align*}
Then, we can get the conclusion of \eqref{eq:vq-vstar}.
\end{proof}

\begin{proof}
By the updated rule, we can obtain that 
\begin{align*}
	\norm{ \hv_k - \tv_k^* }
	\leq \left(1 - \mu \eta\right) \norm{v_k^0 - \tv_k^*}
	\leq 
	\left(1 - \mu \eta\right) \norm{ \hv_{k-1} - \tv_{k-1}^* } + \norm{\tv_{k-1}^* - \tv_k^*}.
\end{align*}
For the second term $\norm{\tv_{k-1}^* - \tv_k^*}$, we have
\begin{align*}
\norm{\tv_{k-1}^* - \tv_k^*}
=&
\norm{ [\nabla_y^2 g(x_{k-1}, \hy_{k-1})]^{-1} \nabla_y f(x_{k-1}, \hy_{k-1}) - [\nabla_y^2 g(x_k, \hy_k)]^{-1} \nabla_y f(x_k, \hy_k) }\\
\leq&
\norm{[\nabla_y^2 g(x_{k-1}, \hy_{k-1})]^{-1} \nabla_y f(x_{k-1}, \hy_{k-1}) - [\nabla_y^2 g(x_k, \hy_k)]^{-1} \nabla_y f(x_{k-1}, \hy_{k-1})}\\
& + \norm{ [\nabla_y^2 g(x_k, \hy_k)]^{-1} \nabla_y f(x_{k-1}, \hy_{k-1}) - [\nabla_y^2 g(x_k, \hy_k)]^{-1} \nabla_y f(x_k, \hy_k)}\\
\leq&
\norm{ [\nabla_y^2 g(x_{k-1}, \hy_{k-1})]^{-1} -  [\nabla_y^2 g(x_k, \hy_k)]^{-1}} \cdot \norm{ \nabla_y f(x_{k-1}, \hy_{k-1}) }\\
&
+ \norm{[\nabla_y^2 g(x_k, \hy_k)]^{-1}  } \norm{\nabla_y f(x_{k-1}, \hy_{k-1}) - \nabla_y f(x_k, \hy_k)}.
\end{align*}
Furthermore, 
\begin{align*}
&\norm{\nabla_y f(x_{k-1}, \hy_{k-1}) - \nabla_y f(x_k, \hy_k)}\\
\leq& 
\norm{ \nabla_y f(x_{k-1}, \hy_{k-1}) - \nabla_y f(x_k, y_k^0) }
+ \norm{ \nabla_y f(x_k, y_k^0) - \nabla_y f(x_k, \hy_k) }\\
\leq&
L \norm{x_{k-1} - x_k} + L \norm{y_k^0 - \hy_k}.
\end{align*}
Then, we have
\begin{align*}
&\norm{ [\nabla_y^2 g(x_{k-1}, \hy_{k-1})]^{-1} -  [\nabla_y^2 g(x_k, \hy_k)]^{-1}} \cdot \norm{ \nabla_y f(x_{k-1}, \hy_{k-1}) }\\
\leq&
\norm{ [\nabla_y^2 g(x_{k-1}, \hy_{k-1})]^{-1} } \norm{\nabla_y^2 g(x_{k-1}, \hy_{k-1}) - \nabla_y^2 g(x_k, \hy_k) } \norm{ [\nabla_y^2 g(x_k, \hy_k)]^{-1} } \\
&\cdot \norm{ \nabla_y f(x_{k-1}, \hy_{k-1}) }\\
\leq&
\frac{\rho \left(\norm{\hy_{k-1}- \hy_k} + { \norm{x_{k-1}- x_k} }\right) }{\mu^2}\norm{ \nabla_y f(x_{k-1}, \hy_{k-1}) }\\
\leq& 
\frac{\rho M }{\mu^2}\left(\norm{\hy_{k-1}- \hy_k} +{ \norm{x_{k-1}- x_k}}\right).
\end{align*}
Thus, we can obtain that
\begin{align*}
\norm{\tv_{k-1}^* - \tv_k^*}
\leq& 
\frac{\rho M \left(\norm{\hy_{k-1}- \hy_k} + \norm{x_{k-1}- x_k}\right) }{\mu^2}
+ \frac{L \norm{x_{k-1} - x_k} + L \norm{y_k^0 - \hy_k}}{\mu}\\
=&
\left(\frac{\rho M}{\mu^2} + \frac{L}{\mu}\right)  \norm{y_k^0 - \hy_k}
+ {\left(\frac{\rho M}{\mu^2} + \frac{L}{\mu} \right)} \norm{x_{k-1} - x_k}.
\end{align*}
Then, we can get the conclusion of \eqref{eq:vq-tvstar}.
\end{proof}

\subsection{Proof of Lemma \ref{lemma-5}}
\begin{proof} 
According to the definition of the hypergradient, we have
\begin{align*}
\norm{\widehat{\nabla} \Phi(x_k) - \nabla \Phi(x_k)}
=& \norm{\nabla_x f(x_k, \hy_k)-\nabla_{xy}^2 g(x_k, \hy_k)\hv_k -
 \nabla_x f(x_k, y_k^*)+\nabla_{xy}^2 g(x_k, y_k^*)v_k^* }
\\
\leq& 
\norm{\nabla_x f(x_k, y_k^*) - \nabla_x f(x_k, \hy_k)} 
+ \norm{\nabla_{xy}^2 g(x_k, \hy_k) (v_k^* - \hv_k)}\\
&+ 
\norm{ \left(\nabla_{xy}^2 g(x_k, y_k^*)- \nabla_{xy}^2 g(x_k, \hy_k) \right) v_k^*  }\\
\leq&
\left(L + \frac{\rho M}{\mu} \right) \norm{\hy_k -y_k^*} + L \norm{\hv_k - v_k^*}\\
\stackrel{\eqref{eq:vq-vstar}}{\leq}&
\left(L + \frac{\rho M}{\mu} + C_0 L\right) \norm{\hy_k -y_k^*} + L \norm{\hv_k - \tv_k^*}.
\end{align*}
Then, the proof is compeleted.
\end{proof}

\subsection{Proof of Lemma \ref{lemma-6}}
\begin{proof}
Firstly, we have
\begin{align*}
	\norm{\hv_k - \tv_k^*} 
	\leq& 
	(1 - \mu \eta) \norm{ \hv_{k-1} - \tv_{k-1}^* } 
	+ C_0 \norm{y_k^0 - \hy_k } + \textcolor{black}{\left( \frac{\rho M}{\mu^2}+\frac{L}{\mu} \right)}\norm{x_{k-1} -x_k}
    \\
	\stackrel{\eqref{eq:yy}}{\leq}& 
	(1 - \mu \eta) \norm{ \hv_{k-1} - \tv_{k-1}^* } 
	+  C_0 \alpha L \norm{\hy_{k-1} - y_{k-1}^*} 
	\\
    &+ {\left(\frac{\alpha L^2 C_0}{\mu} +\frac{\rho M}{\mu^2}+ \frac{L}{\mu} \right) }\norm{x_{k-1} -x_k}.
\end{align*}
Then we have
\begin{align*}
	&\norm{\hv_k - \tv_k^*}
	+ C_1 \norm{\hy_k - y_k^*}\\
	\leq& 
	(1 - \mu \eta) \norm{ \hv_{k-1} - \tv_{k-1}^* } 
	+  C_0 \alpha L \norm{\hy_{k-1} - y_{k-1}^*} 
	+ \textcolor{black}{\left(\frac{\alpha L^2 C_0}{\mu} +\frac{\rho M}{\mu^2}+ \frac{L}{\mu} \right) } \norm{x_{k-1} -x_k}\\
	&+ (1-\mu \alpha) C_1 \norm{\hy_{k-1} - y_{k-1}^*}
	+ \frac{L C_1}{\mu} \norm{x_{k-1} -x_k}\\
	=&
	(1 - \mu \eta) \norm{ \hv_{k-1} - \tv_{k-1}^* }
	+ \left(1 - \mu \alpha + \frac{C_0 \alpha L}{C_1}\right) \cdot C_1 \norm{\hy_{k-1} - y_{k-1}^*}\\
    &+ \textcolor{black}{\left( \frac{\alpha L^2 C_0}{\mu} + \frac{\rho M}{\mu^2}+\frac{L}{\mu} + \frac{LC_1}{\mu} \right)} \norm{x_{k-1} - x_k}.
\end{align*}
By the update rule of $\{x_k\}$, we can obtain that
\begin{align*}
	\norm{x_{k-1} - x_k} 
	=& 
	\beta \norm{\widehat{\nabla} \Phi(x_{k-1})} 
	\leq 
	\beta \norm{\nabla \Phi(x_{k-1})} 
	+ \beta \norm{\widehat{\nabla} \Phi(x_{k-1})-\nabla \Phi(x_{k-1})} \\
	\stackrel{\eqref{eq:nn}}{\leq}&
	\beta \norm{\nabla \Phi(x_{k-1})} 
	+ \beta \left(L + \frac{\rho M}{\mu} + C_0 L\right) \norm{\hy_{k-1} - y_{k-1}^*}
	+ \beta L \norm{\hv_{k-1} - \tv_{k-1}^*}.
	\end{align*}
Thus, we have
\begin{align*}
&\norm{\hv_k - \tv_k^*}
+ C_1 \norm{\hy_k - y_k^*}\\
\leq&
\left( 1 -\mu \eta +  \beta L \textcolor{black}{\left( \frac{\alpha L^2 C_0}{\mu} + \frac{\rho M}{\mu^2}+\frac{L}{\mu} + \frac{LC_1}{\mu} \right)}  \right) \norm{ \hv_{k-1} - \tv_{k-1}^* }\\
&
+ \left(1 - \mu \alpha +  \frac{C_0 \alpha L}{C_1} + \frac{\beta}{C_1} \left(L + \frac{\rho M}{\mu} + C_0 L\right) \textcolor{black}{\left( \frac{\alpha L^2 C_0}{\mu} + \frac{\rho M}{\mu^2}+\frac{L}{\mu} + \frac{LC_1}{\mu} \right)}  \right)  \\
&\cdot C_1  \norm{\hy_{k-1} - y_{k-1}^*}+
\beta  \textcolor{black}{\left( \frac{\alpha L^2 C_0}{\mu} + \frac{\rho M}{\mu^2}+\frac{L}{\mu} + \frac{LC_1}{\mu} \right)}\norm{\nabla \Phi(x_{k-1})}.
\end{align*}
We denote that $C_2= \textcolor{black}{\frac{\alpha L^2 C_0}{\mu} + \frac{\rho M}{\mu^2}+\frac{L}{\mu} + \frac{LC_1}{\mu} }$ and $C_3 = L + \frac{\rho M}{\mu} + C_0 L$. Then the above equation can rewrite as follows
\begin{align*}
&\norm{\hv_k - \tv_k^*}
+ C_1 \norm{\hy_k - y_k^*}\\
\leq&
\left( 1 -\mu \eta +  \beta L C_2  \right) \norm{ \hv_{k-1} - \tv_{k-1}^* }
+ \left(1 - \mu \alpha +  \frac{C_0 \alpha L}{C_1} + \frac{\beta C_2C_3}{C_1}  \right)  \cdot C_1  \norm{\hy_{k-1} - y_{k-1}^*}\\
&+
\beta  C_2 \norm{\nabla \Phi(x_{k-1})}.
\end{align*}
We only need to $1- \mu \eta + \beta L C_2 \leq 1 - \frac{\mu \eta}{2}$, $\frac{C_0 \alpha L}{C_1} = \frac{\mu \alpha}{4}$ and $\frac{\beta C_2 C_3}{C_1} \leq \frac{\mu \alpha}{4}$. Then we can get
\begin{align*}
    & {\beta \leq \frac{\eta \mu}{2LC_2}},\quad{\beta \leq \frac{C_1\mu \alpha}{4C_2C_3}},\quad C_1 = \frac{4C_0 L}{\mu},\\
        &C_2 = {\left( \frac{\alpha L^2 C_0}{\mu} + \frac{\rho M}{\mu^2}+\frac{L}{\mu} + \frac{LC_1}{\mu} \right)\stackrel{\alpha=\frac{1}{L}}{=}\frac{4L^3}{\mu^3} + \frac{4L^2\rho M}{\mu^4} + \frac{L^2}{\mu^2} + \frac{L\rho M}{\mu^3} + \frac{L}{\mu} + \frac{\rho M}{\mu^2}}, \\
	&C_3=L + \frac{\rho M}{\mu} + C_0 L
	= L + \frac{\rho M}{\mu} + \frac{\rho M L}{\mu^2} + \frac{L^2}{\mu}. 
\end{align*}
Then, we have
\begin{align*}
\beta \leq \frac{C_1\mu \alpha}{4C_2C_3} 
=& \frac{\mu^4(\rho M + L\mu)}{(4L^3 \mu + 4L^2\rho M + 2L^2\mu^2 + 2L\mu \rho M +L \mu^3) (L \mu^2 + \rho M\mu + \rho M L +L^2\mu)}\\
=&\mathcal{O}(\kappa^{-4}),\\
\beta \leq \frac{\eta \mu}{2LC_2} \stackrel{\eta=\frac{1}{L}}{=}&
 \frac{\mu^5}{2L^2(4L^3 \mu + 4L^2\rho M + L^2\mu^2 + L\mu \rho M +L \mu^3)} \textcolor{black}{=\mathcal{O}(\kappa^{-5})}.
\end{align*}
\textcolor{black}{Then, we have $\beta\leq\min\{\mathcal{O}(\kappa^{-4}), \mathcal{O}(\kappa^{-5})\}=\mathcal{O}(\kappa^{-5})$}. Thus, we have
\begin{align*}
	&\norm{\hv_k - \tv_k^*}
	+ C_1 \norm{\hy_k - y_k^*}\\
	\leq& 
	\max\{ 1- \frac{\mu \eta}{2}, 1 -\frac{\alpha \mu}{2} \} \cdot \left( \norm{ \hv_{k-1} - \tv_{k-1}^* } + C_1 \norm{\hy_{k-1} - y_{k-1}^*}  \right)
	+ \beta C_2 \norm{\nabla \Phi(x_{k-1})}\\
	=& 
	\left(1 - \frac{\mu}{2L}\right)\cdot \left( \norm{ \hv_{k-1} - \tv_{k-1}^* } + C_1 \norm{\hy_{k-1} - y_{k-1}^*}  \right)
	+ \beta C_2 \norm{\nabla \Phi(x_{k-1})}
\end{align*}
Accordingly, we have
\begin{align*}
\left(\norm{\hv_k - \tv_k^*}
+ C_1 \norm{\hy_k - y_k^*}\right)^2
\leq &
\left(1 - \frac{\mu}{4L}\right) \cdot \left( \norm{ \hv_{k-1} - \tv_{k-1}^* } + C_1 \norm{\hy_{k-1} - y_{k-1}^*}  \right)^2
\\
&+ \frac{3\beta^2 C_2^2 L}{\mu} \norm{\nabla \Phi(x_{k-1})}^2.
\end{align*}
Moreover, we have
\begin{align*}
\norm{\hv_k - \tv_k^*}
+ C_1 \norm{\hy_k - y_k^*}
\leq & 
\left(1 - \frac{\mu}{2L}\right)^k\cdot E_0+\beta C_2 \sum_{t=0}^{k-1} \left(1 - \frac{\mu}{2L}\right)^{k-1-t} \norm{\nabla \Phi(x_t)}.
\end{align*}
Thus, we can obtain that
\begin{align}\label{eq:aid-vy-square}
\left(\norm{\hv_k - \tv_k^*}
+ C_1 \norm{\hy_k - y_k^*}\right)^2
\leq& 
\left(1-\frac{\mu}{4L}\right)^k \cdot E_0^2+
\frac{3\beta^2 C_2^2 L}{\mu}\sum_{t=0}^{k-1} \left(1-\frac{\mu}{4L}\right)^{k-1-t} \norm{\nabla\Phi(x_t)}^2.
\end{align}
Here $E_0=\norm{\hv_0-\tv_0^*}+C_1\norm{\hy_0-y_0^*}$.
Therefore, we have
\begin{align*}
	\norm{\widehat{\nabla}\Phi(x_k) - \nabla \Phi(x_k)}^2
	\leq& 
	 L^2\left( \norm{\hv_k - \tv_k^*} + \frac{C_3}{L} \norm{\hy_k - y_k^*} \right)^2\\
	 \leq&
	 L^2\left( \norm{\hv_k - \tv_k^*} + {C_1} \norm{\hy_k - y_k^*} \right)^2\\
	 \leq&
	 L^2 \left(1-\frac{\mu}{4L}\right)^k \cdot E_0^2+
	 \frac{3\beta^2 C_2^2 L^3}{\mu}\sum_{t=0}^{k-1} \left(1-\frac{\mu}{4L}\right)^{k-1-t} \norm{\nabla\Phi(x_t)}^2,
\end{align*}
where the second inequality is because of $C_3\leq LC_1$ and the specific derivation process is as follows
\begin{align*}
    \frac{C_3}{LC_1}=\frac{\mu(L+\mu)}{4L^2}\leq\frac{L(L+L)}{4L^2}<\frac{1}{2},
\end{align*}
where $C_3=\frac{(L\mu+\rho M)\cdot(L+\mu)}{\mu^2}$ and $LC_1=\frac{4L^2(L\mu+\rho M)}{\mu^3}$.
\end{proof}
\subsection{Proof of Lemma \ref{lem:aid-init}}
\begin{proof}

Let $y_{0}^{(0)}=y_{0,\mathrm{raw}}$ and perform the warm-up recursion
$
y_{0}^{(s+1)}=y_{0}^{(s)}-\frac{1}{L}\nabla_y g(x_0,y_{0}^{(s)})
$
for $s=0,\ldots,N_0-1$, and set $y_0=y_{0}^{(N_0)}$ as the initial point used in Algorithm~\ref{alg:bo-aid}. Since $g(x_0,\cdot)$ is $\mu$-strongly convex and $L$-smooth, gradient descent with stepsize $1/L$ satisfies

\begin{align*}
\norm{y_0-y_0^*}
\leq
\left(1-\frac{\mu}{L}\right)^{N_0}D_y
\leq c_y\mu .
\end{align*}

The first AID inner update is one additional gradient descent step at the same $x_0$, hence

\begin{align*}
\norm{\hy_0-y_0^*}
\leq
\left(1-\frac{\mu}{L}\right)\norm{y_0-y_0^*}
\leq c_y\mu .
\end{align*}

Next, since $v_0=0$ and $\eta=1/L$, the first linear-system iterate satisfies
$
\hv_0=\eta\nabla_y f(x_0,\hy_0)
$
and Assumption~\ref{assum-2} gives $\norm{\nabla_y f(x_0,\hy_0)}\leq M$. Therefore

\begin{align*}
\norm{\hv_0}\leq \frac{M}{L}.
\end{align*}

Moreover, by the definition of $\tv_0^*$ and the $\mu$-strong convexity of $g(x_0,\cdot)$,

\begin{align*}
\norm{\tv_0^*}
=
\norm{[\nabla_y^2 g(x_0,\hy_0)]^{-1}\nabla_y f(x_0,\hy_0)}
\leq \frac{M}{\mu}.
\end{align*}

Thus

\begin{align*}
\norm{\hv_0-\tv_0^*}
\leq
\frac{M}{L}+\frac{M}{\mu}.
\end{align*}

Combining the above estimate with $C_1=4C_0L/\mu$ and
$C_0=\rho M/\mu^2+L/\mu$, we obtain

{
\begin{align*}
E_0
=
\norm{\hv_0-\tv_0^*}+C_1\norm{\hy_0-y_0^*}\leq
\frac{M}{L}+\frac{M}{\mu}
+\frac{4C_0L}{\mu}\cdot c_y\mu=
\frac{M}{L}+\frac{M}{\mu}
+4c_yL\left(\frac{\rho M}{\mu^2}+\frac{L}{\mu}\right).
\end{align*}
}

Under the same condition-number bookkeeping used throughout the paper, namely treating $L,M,\rho$ as constants and $\kappa=L/\mu$, the right-hand side is $\gO(\kappa^2)$. Hence $E_0^2=\gO(\kappa^4)$, and since $L^3/\mu=\gO(\kappa)$, we conclude that

\begin{align*}
\frac{L^3}{\mu}E_0^2=\gO(\kappa^5).
\end{align*}
This completes the proof.
\end{proof}
\subsection{Proof of Theorem \ref{theo:aid}}
\begin{proof}
First, based on Lemma 2 in \cite{ji2021bilevel}, we have $\nabla \Phi(\cdot)$ is $L_\Phi$-Lipschitz, where $L_\Phi= L + \frac{2L^2+\rho M^2}{\mu} + \frac{2\rho L M+L^3}{\mu^2} + \frac{\rho L^2 M}{\mu^3}=\Theta(\kappa^3)$. Then, we have 
\begin{small}
\begin{align*}
\Phi(x_{k+1})
\leq{}&\Phi(x_k)+\dotprod{\nabla\Phi(x_k),x_{k+1}-x_k}+\frac{L_\Phi}{2}\norm{x_{k+1}-x_k}^2\\
\leq{}&\Phi(x_k)-\left(\frac{\beta}{2}-\beta^2L_\Phi\right)
\norm{\nabla\Phi(x_k)}^2+\left(\frac{\beta}{2}+\beta^2L_\Phi\right)
\norm{\nabla\Phi(x_k)-\widehat{\nabla}\Phi(x_k)}^2\\
\stackrel{\eqref{eq:pp}}{\leq}{}&\Phi(x_k)
-\left(\frac{\beta}{2}-\beta^2L_\Phi\right)
\norm{\nabla\Phi(x_k)}^2+\left(\frac{\beta}{2}+\beta^2L_\Phi\right)L^2
\left(1-\frac{\mu}{4L}\right)^k E_0^2\\
&+\left(\frac{\beta}{2}+\beta^2L_\Phi\right)
\frac{3\beta^2C_2^2L^3}{\mu}\cdot\sum_{t=0}^{k-1}
\left(1-\frac{\mu}{4L}\right)^{k-1-t}
\norm{\nabla\Phi(x_t)}^2.
\end{align*}  
\end{small}
Define
\begin{align*}
A={}&\frac{1}{2}-\beta L_\Phi-\left(\frac{1}{2}+\beta L_\Phi\right)\frac{12\beta^2C_2^2L^4}{\mu^2}.
\end{align*}
Telescoping the preceding inequality over $k=0,\ldots,K-1$ gives
\begin{align*}
\Phi(x_K)
\leq{}&\Phi(x_0)-\beta A
\sum_{k=0}^{K-1}\norm{\nabla\Phi(x_k)}^2+\left(\frac{\beta}{2}+\beta^2L_\Phi\right)
\frac{4L^3}{\mu}E_0^2.
\end{align*}
The choice of $\beta$ gives $\beta L_\Phi\leq1/8$ and
$12\beta^2C_2^2L^4/\mu^2\leq3/16$, and hence $A\geq33/128$.
Moreover $\beta=\Theta(\kappa^{-5})$. Therefore, we can obtain that
\begin{align*}
    \frac{1}{K}\sum_{k=0}^{K-1} \norm{\nabla \Phi(x_k)}^2 \leq
    \frac{\Phi(x_0)-\Phi^\star}{\beta AK}+\frac{4L^3(1+2\beta L_\Phi)}{2\mu AK}\cdot E_0^2,
\end{align*}
where $L_{\Phi}=L+\frac{2L^2+\rho M^2}{\mu}+\frac{2\rho LM+L^3}{\mu^2}+\frac{\rho L^2 M}{\mu^3}=\gO(\kappa^3)$. By Lemma~\ref{lem:aid-init}, $\frac{L^3}{\mu}E_0^2=\gO(\kappa^5)$. For the first term, we have
\begin{align*}
    \frac{\Phi(x_0)-\Phi^\star}{\beta A}=\gO(\kappa^{5}).
\end{align*}
For the second term, we have
\begin{align*}
    &\frac{4L^3(1+2\beta L_\Phi)}{2\mu AK}\cdot E_0^2
    =
    \frac{4L^3(1+2\beta L_\Phi)}{2\mu A K}\cdot E_0^2=\gO(\kappa^5/K).
\end{align*}
Then, we have
\begin{align*}
    \frac{1}{K}\sum_{k=0}^{K-1} \norm{\nabla \Phi(x_k)}^2 = \gO\left(\frac{\kappa^5}{K}+\frac{\kappa^5}{K}
    \right)=\gO\left(\frac{\kappa^5}{K}
    \right).
\end{align*}
Then, to achieve an $\epsilon$-accurate stationary point, we have $K=\gO(\kappa^5\epsilon^{-1})$, and hence we have the following complexity results. 1) Gradient complexity: $\text{Gc}(\epsilon)=3K=\widetilde{\gO}(\kappa^5\epsilon^{-1})$. 2) Matrix-vector product complexities: $\text{Mv}(\epsilon)=K+KQ=\widetilde{\gO}(\kappa^5\epsilon^{-1})$. 
\end{proof}
\section{Proofs of the single-loop ITD-based algorithm}
\subsection{Additional useful Lemma}
\begin{lemma}\label{lemma:ghy-gystar}
    Consider the single-loop ITD-based algorithm in Algorithm \ref{alg:bo-itd}. Suppose Assumptions \ref{assum-1}-\ref{assum-4} hold. Let $\alpha\leq\frac{1}{L}$, we have
    \begin{align*}
        \norm{\nabla_x y_{k}^N(x_k)-\nabla_x y_{k}^*(x_k)}\leq (1-\alpha\mu)\norm{\nabla_x y^*(x_k)}+\alpha\rho\norm{y_k^0-y^*(x_k)},
    \end{align*}
    where $y_{k}^N(x_k)=y_k^0-\alpha\nabla_y g(x_k, y_k^0)$ and $y_k^*=\argmin_y g(x_k, y)$ for $k=1,\dots,K$.
\end{lemma}
\begin{proof}
Accordingly,
\begin{align*}
\nabla_x y_k^N(x_k)
&=-\alpha\nabla_{xy}^2g(x_k,y_k^0),\\
\nabla_x y^*(x_k)
&=-[\nabla_{yy}^2g(x_k,y_k^*)]^{-1}
\nabla_{xy}^2g(x_k,y_k^*).
\end{align*}
Thus,
\begin{align*}
    \norm{\nabla_x y_{k}^N(x_k)-\nabla_x y_{k}^*(x_k)}
    =&\norm{-\alpha\nabla_{xy}^2g(x_k,y_k^0)+[\nabla_{yy}^2 g(x_k, y^*_k)]^{-1}\nabla_{xy}^2 g(x_k, y_k^*)}\\
    \leq &
    \norm{\left(I-\alpha\nabla_{yy}^2 g(x_k, y^*_k)\right)[\nabla_{yy}^2 g(x_k, y^*_k)]^{-1}\nabla_{xy}^2 g(x_k, y_k^*)}\\
    &+\norm{\alpha\left( \nabla_{xy}^2 g(x_k, y_k^*) -\nabla_{xy}^2g(x_k,y_k^0) \right)}\\
    \leq &
    (1-\alpha\mu)\norm{\nabla_x y^*(x_k)}+\alpha\rho\norm{y_k^0-y^*(x_k)}.
\end{align*}
Then, the proof is completed.
\end{proof}

\subsection{Proof of Lemma \ref{lem:itd-radius}}
\begin{proof}
Let $a_k=\norm{y_k^0-y_k^*}$. Since ITD uses the warm-start rule $y_{k+1}^0=\hy_k$ and $\hy_k=y_k^0-\alpha\nabla_y g(x_k,y_k^0)$, the contraction of one gradient-descent step on the $\mu$-strongly convex and $L$-smooth function $g(x_k,\cdot)$ gives
{
\begin{align*}
\norm{y_{k+1}^0-y_k^*}
=\norm{\hy_k-y_k^*}
\leq (1-\alpha\mu)a_k .
\end{align*}
}
Moreover, the solution map $y^*(x)$ is $L/\mu$-Lipschitz under Assumptions~\ref{assum-1} and~\ref{assum-3}. Therefore,
{
\begin{align*}
a_{k+1}
&=\norm{y_{k+1}^0-y_{k+1}^*}  \\
&\leq \norm{y_{k+1}^0-y_k^*}+\norm{y_k^*-y_{k+1}^*} \\
&\leq (1-\alpha\mu)a_k+\frac{L}{\mu}\norm{x_{k+1}-x_k}.
\end{align*}
}
It remains to bound the outer displacement. By Assumption~\ref{assum-2}, $\norm{\nabla_x f(x_k,\hy_k)}\leq M$ and $\norm{\nabla_y f(x_k,\hy_k)}\leq M$; by Assumption~\ref{assum-3}, $\norm{\nabla_{xy}^2g(x_k,y_k^0)}\leq L$. Hence the ITD hypergradient estimator satisfies
{
\begin{align*}
\norm{\widehat\nabla\Phi(x_k)}
&\leq \norm{\nabla_x f(x_k,\hy_k)}
    +\alpha\norm{\nabla_{xy}^2g(x_k,y_k^0)}
       \norm{\nabla_y f(x_k,\hy_k)}\\
&\leq M(1+\alpha L)=G.
\end{align*}
}
Since $x_{k+1}=x_k-\beta\widehat\nabla\Phi(x_k)$, we obtain the radius recurrence
{
\begin{align*}
a_{k+1}\leq (1-\alpha\mu)a_k+\frac{L\beta G}{\mu}.
\end{align*}
}
If $a_k\leq R_y$ and $\beta\leq \alpha\mu^2R_y/(LG)$, then
{
\begin{align*}
a_{k+1}
\leq (1-\alpha\mu)R_y+\alpha\mu R_y
=R_y.
\end{align*}
}
The induction starts from $a_0\leq R_y$, so $a_k\leq R_y$ for all $k\geq0$.

For the parameters used in Theorem~\ref{theo:itd}, $\alpha=1/L$ and thus $G=2M$. The theorem's stepsize bound implies the radius condition whenever
{
\begin{align*}
\frac{\mu^3}{2L(2L^2+\rho M)}
\leq \frac{\mu^2R_y}{2L^2M},
\end{align*}
}
which is equivalent to $R_y\geq \mu L M/(2L^2+\rho M)$. Finally, if $y_{0,\mathrm{raw}}$ is any initial point with $D_y=\norm{y_{0,\mathrm{raw}}-y_0^*}<\infty$, then $N_0^{\rm itd}$ warm-up steps on $g(x_0,\cdot)$ with stepsize $1/L$ yield
{
\begin{align*}
\norm{y_0-y_0^*}
\leq \left(1-\frac{\mu}{L}\right)^{N_0^{\rm itd}}D_y.
\end{align*}
}
Thus $a_0\leq R_y$ is guaranteed as soon as
{
\begin{align*}
N_0^{\rm itd}
\geq
\left\lceil
\frac{\log(D_y/R_y)_+}{\log\left(1/(1-\mu/L)\right)}
\right\rceil,
\end{align*}
}
where $\log(t)_+=\max\{\log(t),0\}$. Under the condition-number bookkeeping used in the paper, $L,M,\rho=\gO(1)$ and $\kappa=L/\mu$, so any fixed constant $R_y$ satisfying the displayed lower bound is $R_y=\gO(1)$.
\end{proof}

\subsection{Proof of Lemma \ref{lemma:hy-ystar}}
\begin{proof}
Accordingly, we have
\begin{small}
 \begin{align*}
        \norm{\hy_k-y^*(x_k)}\leq & (1-\mu\alpha)\norm{\hy_{k-1}-y^*(x_{k-1})} + \frac{L}{\mu}\norm{x_{k-1}-x_k} \\
        \leq & (1-\mu\alpha)\norm{\hy_{k-1}-y^*(x_{k-1})} + \frac{L\beta}{\mu}\norm{\widehat{\nabla}\Phi (x_{k-1})} \\
        \leq & (1-\mu\alpha)\norm{\hy_{k-1}-y^*(x_{k-1})} + \frac{L\beta}{\mu}\left(
        \norm{{\nabla}\Phi (x_{k-1})} + \norm{\widehat{\nabla}\Phi (x_{k-1})-{\nabla}\Phi (x_{k-1})}
        \right) \\
        \leq & (1-\mu\alpha)\norm{\hy_{k-1}-y^*(x_{k-1})} + \frac{L\beta}{\mu}\left(
        \norm{{\nabla}\Phi (x_{k-1})} + C_4\norm{\hy_{k-1}-y^*_{k-1}} + C_5
        \right) \\
        \leq & \left(1-\mu\alpha + \frac{L\beta C_4}{\mu}\right)\norm{\hy_{k-1}-y^*(x_{k-1})} + \frac{L\beta}{\mu}
        \norm{{\nabla}\Phi (x_{k-1})} + \frac{L\beta C_5}{\mu}.
    \end{align*}   
\end{small}
    We rewrite the above equation as $\norm{\hy_k-y^*(x_k)}\leq C_6\norm{\hy_{k-1}-y^*(x_{k-1})} + \frac{L\beta}{\mu}\norm{{\nabla}\Phi (x_{k-1})}+C_7$, where $C_6=1-\mu\alpha + \frac{L\beta C_4}{\mu}$ and $C_7=\frac{L\beta C_5}{\mu}$. Then, the proof of \eqref{eq:itd:yhat-ystar} is completed.

    Since $\alpha=\frac{1}{L}$ and $\beta\leq\frac{\mu^2}{16L^2C_4}$, we have $\frac{L\beta C_4}{\mu}\leq\frac{\mu}{16L}\leq\frac{\mu}{2L}$ and hence $C_6\leq1-\frac{\mu}{2L}$. Accordingly, we have
\begin{align*}
    \norm{\hy_k-y^*(x_k)}\leq \left(1-\frac{\mu}{2L}\right)^k\norm{\hy_0-y^*(x_0)} + \frac{L\beta}{\mu}\sum_{j=0}^{k-1}\left(1-\frac{\mu}{2L}\right)^{k-1-j}\left(\norm{\nabla\Phi(x_j)}+C_5\right).
\end{align*}
Then, the proof of \eqref{eq:itd:yhat-ystar-v2} is completed. Similar with AID in \eqref{eq:aid-vy-square}, we can obtain
\begin{align}\label{eq:y_hat-ystar-square}
    \norm{\hy_k-y^*(x_k)}^2\leq \left(1-\frac{\mu}{4L}\right)^k\norm{\hy_0-y^*(x_0)}^2 + \frac{3L^3\beta^2}{\mu^3}\sum_{j=0}^{k-1}\left(1-\frac{\mu}{4L}\right)^{k-1-j}\left(\norm{\nabla\Phi(x_j)}+C_5\right)^2.
\end{align}
\end{proof}

\subsection{Technical Lemma}
\begin{lemma}[Contraction of Error Sequences]
\label{lem:error_contraction}
Let $\{a_k\}_{k\geq 0}$ and $\{b_k\}_{k\geq 0}$ be two non-negative sequences satisfying the following recursive inequalities:
\begin{align}
    a_k &\leq (1-\alpha\mu)a_{k-1} + \frac{\alpha\rho}{2}a_{k-1}^2 + \frac{L\beta\rho}{\mu}b_{k-1} + \frac{L\beta c'}{\mu}, \label{eq:a_recursion} \\
    b_k &\leq (1-\eta\mu+c_0\beta\rho)b_{k-1} + \alpha c_0 L a_k + c_0\beta c', \label{eq:b_recursion}
\end{align}
where $\alpha \leq 1/L$, $c_0 = \frac{\rho M}{\mu^2} + \frac{L}{\mu}$, and $c' = M + \rho\frac{M}{\mu}$. Define $R_a = \frac{\mu}{2\rho}$ and $R_b = \frac{2\alpha c_0 L}{\eta \rho}$. If the initial errors satisfy $a_0 \leq R_a$ and $b_0 \leq R_b$, and the parameter $\beta$ satisfies:
\begin{equation}
    \beta \leq \min \left\{ \frac{\alpha\mu^3}{8\rho L c'}, \quad \frac{\mu^3 \eta}{16 \rho L^2 c_0}, \quad \frac{\eta\mu}{2c_0\rho}, \quad \frac{\alpha\mu L}{2\rho c'} \right\},
\end{equation}
then for all $k \geq 0$, we have $a_k \leq R_a$ and $b_k \leq R_b$.
\end{lemma}

\begin{proof}
We proceed by induction on $k$. The base case $k=0$ holds by the lemma's assumptions. Now, assume $a_{k-1} \leq R_a$ and $b_{k-1} \leq R_b$ for some $k \geq 1$.

\textbf{Step 1: Bounds for $a_k$.} 
Substituting the inductive hypothesis $a_{k-1} \leq R_a = \frac{\mu}{2\rho}$ into the quadratic term of \eqref{eq:a_recursion}:
\begin{equation*}
    \frac{\alpha\rho}{2}a_{k-1}^2 \leq \left( \frac{\alpha\rho}{2} \cdot \frac{\mu}{2\rho} \right) a_{k-1} = \frac{\alpha\mu}{4} a_{k-1}.
\end{equation*}
Then \eqref{eq:a_recursion} becomes $a_k \leq (1 - \frac{3}{4}\alpha\mu)a_{k-1} + \frac{L\beta\rho}{\mu}b_{k-1} + \frac{L\beta c'}{\mu}$. To ensure $a_k \leq R_a$, it suffices to show the remaining terms are bounded by $\frac{1}{2}\alpha\mu R_a$:
\begin{enumerate}
    \item Using $\beta \leq \frac{\mu^3 \eta}{16 \rho L^2 c_0}$, we have: $\frac{L\beta\rho}{\mu}b_{k-1} \leq \frac{L\rho}{\mu} \cdot \frac{2\alpha c_0 L}{\eta \rho} \cdot \frac{\mu^3 \eta}{16 \rho L^2 c_0} = \frac{\alpha\mu^2}{8\rho} = \frac{1}{4}\alpha\mu R_a$.
    \item Using $\beta \leq \frac{\alpha\mu^3}{8\rho L c'}$, we have: $\frac{L\beta c'}{\mu} \leq \frac{L c'}{\mu} \cdot \frac{\alpha\mu^3}{8\rho L c'} = \frac{\alpha\mu^2}{8\rho} = \frac{1}{4}\alpha\mu R_a$.
\end{enumerate}
Summing these yields $a_k \leq (1 - \frac{3}{4}\alpha\mu)R_a + \frac{1}{2}\alpha\mu R_a = (1 - \frac{1}{4}\alpha\mu)R_a \leq R_a$.

\textbf{Step 2: Bounds for $b_k$.}
Using $\beta \leq \frac{\eta\mu}{2c_0\rho}$, the coefficient in \eqref{eq:b_recursion} satisfies $(1 - \eta\mu + c_0\beta\rho) \leq (1 - \frac{1}{2}\eta\mu)$. Then:
\begin{equation*}
    b_k \leq (1 - \frac{1}{2}\eta\mu) R_b + \alpha c_0 L a_k + c_0\beta c'.
\end{equation*}
To ensure $b_k \leq R_b$, we bound the coupling and constant terms by $\frac{1}{2}\eta\mu R_b$:
\begin{enumerate}
    \item Since $a_k \leq R_a$, $\alpha c_0 L a_k \leq \alpha c_0 L \frac{\mu}{2\rho} = \frac{\alpha c_0 L \mu}{2\rho} = \frac{1}{4}\eta\mu R_b$ (by definition of $R_b$).
    \item Using $\beta \leq \frac{\alpha\mu L}{2\rho c'}$, $c_0\beta c' \leq c_0 c' \frac{\alpha\mu L}{2\rho c'} = \frac{\alpha c_0 L \mu}{2\rho} = \frac{1}{4}\eta\mu R_b$.
\end{enumerate}
Summing these yields $b_k \leq (1 - \frac{1}{2}\eta\mu) R_b + \frac{1}{2}\eta\mu R_b = R_b$. This completes the induction.
\end{proof}
\subsection{Proof of Lemma \ref{lem:hPhi-Phi}}
\begin{proof}
First, according to the definition of $\widehat{\nabla}\Phi(x_k)$ and $\nabla\Phi(x_k)$, we have
    \begin{align*}
        &\norm{\widehat{\nabla}\Phi(x_k)-\nabla\Phi(x_k)}\\
        \leq &
    \norm{\nabla_1 f(x_k, \hy_k)+\nabla_x \hy_k(x_k)\nabla_2 f(x_k, \hy_k)-
    \nabla_1 f(x_k, y_k^*)-\nabla_x y_k^*(x_k)\nabla_2 f(x_k, y_k^*)}
    \\
    \leq & L\norm{\hy_k-y^*_k} + \norm{\nabla_x \hy_k(x_k)\nabla_2 f(x_k, \hy_k) -\nabla_x \hy_k(x_k)\nabla_2 f(x_k, y_k^*)} \\
    & + \norm{\nabla_x \hy_k(x_k)\nabla_2 f(x_k, y_k^*)-\nabla_x y_k^*(x_k)\nabla_2 f(x_k, y_k^*)} \\
    \leq & L\norm{\hy_k-y^*_k} + \alpha L^2\norm{\hy_k-y_k^*} +M\left(
    (1-\alpha\mu)\frac{L}{\mu} +\alpha\rho\norm{y_k^0-y_k^*}
    \right).
    \end{align*}
Using Lemma~\ref{lem:itd-radius}, $\norm{y_k^0-y_k^*}\leq R_y$, and the truncation term is absorbed into a uniform constant.
Then, we have
\begin{align}
     \norm{\widehat{\nabla}\Phi(x_k)-\nabla\Phi(x_k)}\leq C_4\norm{\hy_k-y^*_k} + C_5,
\end{align}
where $C_4=L+\alpha L^2 + \alpha\rho M$ and $C_5=M(1-\alpha\mu)\frac{L}{\mu}+\alpha\rho M R_y$.
\end{proof}

\subsection{Proof of Lemma \ref{lemma:hgPhi-gPhi-itd}}
\begin{proof}
According to Lemma \ref{lem:hPhi-Phi}, we have
    \begin{align*}
        &\norm{\widehat{\nabla}\Phi(x_k)-\nabla\Phi(x_k)}^2\leq \left(C_4\norm{\hy_k-y^*_k} + C_5\right)^2\leq 2C_4^2\norm{\hy_k-y^*_k}^2 + 2C_5^2\\
        \leq & 
        2C_4^2\left(1-\frac{\mu}{4L}\right)^k\norm{\hy_0-y^*_0}^2 + \frac{6L^3\beta^2C_4^2}{\mu^3}\sum_{j=0}^{k-1}\left(1-\frac{\mu}{4L}\right)^{k-1-j}\left(\norm{\nabla\Phi(x_j)}+C_5\right)^2 + 2C_5^2,
    \end{align*}
where the last inequality holds since \eqref{eq:y_hat-ystar-square}.
\end{proof}
\subsection{Proof of Theorem \ref{theo:itd}}
\begin{proof}
First, based on Lemma 2 in \cite{ji2021bilevel}, we have $\nabla \Phi(\cdot)$ is $L_\Phi$-Lipschitz, where $L_\Phi= L + \frac{2L^2+\rho M^2}{\mu} + \frac{2\rho L M+L^3}{\mu^2} + \frac{\rho L^2 M}{\mu^3}=\Theta(\kappa^3)$. Then, we have 
\begin{small}
    \begin{align*}
        \Phi(x_{k+1}) \leq & \Phi(x_{k}) -\left(\frac{\beta}{2}-\beta^2 L_{\Phi}\right)\norm{\nabla\Phi(x_k)}^2+\left(\frac{\beta}{2}+\beta^2 L_{\Phi}\right)\norm{\widehat{\nabla}\Phi(x_k)-\nabla\Phi(x_k)}^2\\
        \leq & \Phi(x_{k}) -\left(\frac{\beta}{2}-\beta^2 L_{\Phi}\right)\norm{\nabla\Phi(x_k)}^2+\left(\frac{\beta}{2}+\beta^2 L_{\Phi}\right)2C_4^2\left(1-\frac{\mu}{4L}\right)^k\norm{\hy_0-y^*_0}^2\\
        & + \left(\frac{\beta}{2}+\beta^2 L_{\Phi}\right)\frac{6L^3\beta^2C_4^2}{\mu^3}\sum_{j=0}^{k-1}\left(1-\frac{\mu}{4L}\right)^{k-1-j}\left(\norm{\nabla\Phi(x_j)}+C_5\right)^2 + \left(\frac{\beta}{2}+\beta^2 L_{\Phi}\right)2C_5^2.
    \end{align*}
    \end{small}
    Telescoping the above equation over $k$ from $0$ to $K-1$ yields
    \begin{align*}
        \Phi(x_{K}) \leq & \Phi(x_{0}) -\left(\frac{\beta}{2}-\beta^2 L_{\Phi}\right)\sum_{k=0}^{K-1}\norm{\nabla\Phi(x_k)}^2+\left(\frac{\beta}{2}+\beta^2 L_{\Phi}\right)C_4^2\frac{8L}{\mu}\norm{\hy_0-y^*_0}^2\\
        & + \left(\frac{\beta}{2}+\beta^2 L_{\Phi}\right)\frac{6L^3\beta^2C_4^2}{\mu^3}\frac{4L}{\mu}\sum_{k=0}^{K-1}\left(\norm{\nabla\Phi(x_k)}+C_5\right)^2 + \left(\frac{\beta}{2}+\beta^2 L_{\Phi}\right)2C_5^2K\\
        \leq & \Phi(x_{0})-A\sum_{k=0}^{K-1}\norm{\nabla\Phi(x_k)}^2+B_1+B_2K+\left(\frac{\beta}{2}+\beta^2 L_{\Phi}\right)2C_5^2K,
    \end{align*}
    where 
    \begin{align*}
        &A=\left(\frac{\beta}{2}-\beta^2 L_{\Phi}\right)-\left(\frac{\beta}{2}+\beta^2 L_{\Phi}\right)\frac{48L^4\beta^2 C_4^2}{\mu^4},\quad B_1=\left(\frac{\beta}{2}+\beta^2 L_{\Phi}\right)C_4^2\frac{8L}{\mu}\norm{\hy_0-y^*_0}^2,
        \\
        &B_2=\left(\frac{\beta}{2}+\beta^2 L_{\Phi}\right)\frac{48L^4\beta^2 C_4^2C_5^2}{\mu^4}.
    \end{align*}
    The last inequality uses $(a+b)^2\leq2a^2+2b^2$. By the choice of $\beta$, $\beta L_\Phi\leq1/8$ and $\frac{48L^4\beta^2C_4^2}{\mu^4}\leq3/16$, hence $A\geq33\beta/128>0$. Thus we have
    \begin{align*}
        \frac{1}{K}\sum^{K-1}_{k=0}\norm{\nabla\Phi(x_k)}^2\leq\frac{\Phi(x_0)-\Phi^\star}{AK}+\frac{B_1}{AK}+\frac{B_2}{A}+\left(\frac{\beta}{2}+\beta^2 L_{\Phi}\right)\frac{2C_5^2}{A},
    \end{align*}
    where $\beta=\Theta(\kappa^{-3})$, $L_\Phi =\gO(\kappa^{3}),\quad C_4=\gO(1),\quad C_5 =\gO(\kappa^{1})$. Thus we have $\frac{1}{A}=\gO(\kappa^{3})$, $\frac{B_1}{A}=\gO(\kappa^{1})$, and $\frac{B_2}{A}=\gO(1)$. Moreover, $\left(\frac{\beta}{2}+\beta^2L_\Phi\right)/A=\gO(1)$, so the final non-vanishing term is $\gO(C_5^2)=\gO(\kappa^2)$. Therefore, we have
    \begin{align*}
        \frac{1}{K}\sum^{K-1}_{k=0}\norm{\nabla\Phi(x_k)}^2=\gO\left(\frac{\kappa^3}{K}+\kappa^2\right).
    \end{align*}
Therefore the proof is completed.
\end{proof}

% \section{Additional experiments}\label{sec:add-exp}
\section{Experimental Settings of Real Data}\label{sec:exp-set}

In our experiments, we consider two widely used tasks, feature learning \citep{franceschi2018bilevel} and data reweighting for noisy labels \citep{shaban2019truncated}. 
% Please refer to Section 2 for the formulation of the two tasks.p

For feature learning \citep{franceschi2018bilevel, bao2021stability}, we evaluate the regression problem on the space dataset \citep{CC01a}. We randomly select $500$ and $500$ images for training and validation respectively. The outer variable $x$ represents the parameters in a linear layer of size $6\rightarrow128$ to extract features. The inner variable $y$ represents the parameters in a linear layer of size $128\rightarrow1$ to predict the value. The total iteration size is $10000$. 

For data reweighting, we evaluate a classification problem on the MNIST dataset following \citet{shaban2019truncated, bao2021stability}. MNIST consists of grayscale handwritten digits of size $28\times 28$. We randomly select 2000 and 500 images for training and validation respectively. The label of a training sample is replaced by
a uniformly sampled wrong label with probability $0.2$. $x$ represents the logits of the weights of the
training data, and $y$ represents the parameters in an MLP of size $784\rightarrow256\rightarrow10$. The learning rate is $1e^{-3}$ in the outer-level and $5e^{-2}$ in the inner-level.  The total iteration size is $3000$.

\end{document}

%% file: math_commands.tex
\usepackage{amsmath,amsfonts,bm, amsthm}

\newcommand{\dotprod}[1]{\left\langle #1\right\rangle}

\newcommand{\tv}{\tilde{v}}
\newcommand{\hv}{\hat{v}}
\newcommand{\hy}{\hat{y}}
\newcommand{\norm}[1]{\left\|#1\right\|}
\def\eqref#1{Eq.~(\ref{#1})}
\def\1{\bm{1}}

\DeclareMathAlphabet{\mathsfit}{\encodingdefault}{\sfdefault}{m}{sl}
\SetMathAlphabet{\mathsfit}{bold}{\encodingdefault}{\sfdefault}{bx}{n}

\def\gO{{\mathcal{O}}}

\def\sR{{\mathbb{R}}}

\DeclareMathOperator*{\argmin}{arg\,min}